\documentclass{article}
\usepackage{arxiv}

\usepackage[english]{babel}
\usepackage{natbib} 
\usepackage[utf8]{inputenc} % allow utf-8 input
\usepackage[T1]{fontenc}    % use 8-bit T1 fonts
\usepackage[colorlinks=true, allcolors=blue]{hyperref}
\newcommand{\email}[1]{\href{mailto:#1}{\tt{\nolinkurl{#1}}}}
\newcommand{\orcid}[1]{ORCID: \href{https://orcid.org/#1}{\tt{\nolinkurl{#1}}}}
\usepackage{authblk}
\usepackage{xcolor}
\usepackage{hyperref}
\usepackage{url}
\usepackage{mlmath}
\usepackage{array}
\usepackage{tabularx} % 加载tabularx宏包
\usepackage{amsmath,amsfonts,amssymb,mathtools,amsthm}

\newtheorem*{theorem*}{Theorem}
\newtheorem{lemma}{Lemma}
\newtheorem*{lemma*}{Lemma}

\newtheorem*{property*}{Property}
\newtheorem{remark}{Remark}
\newtheorem{definition}{Definition}
\newtheorem{corollary}{Corollary}

\newtheorem*{assumption*}{Assumption}
\newtheorem*{prop*}{Proposition}
\newtheorem{prop}{Proposition}

\newtheorem*{setting*}{Setting}

\definecolor{mycolor}{HTML}{FFE9D1}
\title{Probability Signature: Bridging Data Semantics and Embedding Structure in Language Models}

\author[1,2]{Junjie Yao}
\author[1,2,*]{Zhi-Qin John Xu}
\affil[1]{Institute of Natural Sciences, MOE-LSC, Shanghai Jiao Tong University}
\affil[2]{School of Mathematical Sciences, Shanghai Jiao Tong University}
\affil[*]{Corresponding author}

\begin{document}

\maketitle

\begin{abstract}
The embedding space of language models is widely believed to capture the semantic relationships; for instance, embeddings of digits often exhibit an ordered structure that corresponds to their natural sequence. However, the mechanisms driving the formation of such structures remain poorly understood. In this work, we interpret the embedding structures via the data distribution. We propose a set of probability signatures that reflect the semantic relationships among tokens. Through experiments on the composite addition tasks using the linear model and feedforward network, combined with theoretical analysis of gradient flow dynamics, we reveal that these probability signatures significantly influence the embedding structures. We further generalize our analysis to large language models (LLMs) by training the Qwen2.5 architecture on the subsets of the Pile corpus. Our results show that the probability signatures are faithfully aligned with the embedding structures, particularly in capturing strong pairwise similarities among embeddings. Our work uncovers the mechanism of how data distribution guides the formation of embedding structures, establishing a novel understanding of the relationship between embedding organization and semantic patterns.
\end{abstract}

\section{Introduction}
In recent years, deep neural network-based large language models (LLMs) have demonstrated remarkable performance~\citep{comanici2025gemini25pushingfrontier,openai2024gpt4technicalreport,deepseekai2025deepseekr1}. The development of these models has largely followed what Richard Sutton termed ``the bitter lesson''--that the most effective approach to improving AI performance has historically been to leverage greater computational resources, larger models, and more data, rather than incorporating human knowledge or specialized architectures~\citep{sutton2019bitter}. This trend has been formalized through scaling laws, which quantify the relationship between model performance and factors such as model size, dataset size, and computational budget through power law relationships~\citep{kaplan2020scaling}. While these scaling laws provide valuable quantitative predictions for model performance, they also reveal a concerning limitation: the power law relationship suggests that achieving further significant improvements may require prohibitively large increases in model and data size, making continued scaling increasingly impractical and resource-intensive. 

One promising approach to address these limitations is to develop a deeper understanding of the underlying mechanisms that drive transformer models' success in natural language processing (NLP). The No Free Lunch theorem establishes that no single algorithm can perform optimally across all problem domains, highlighting the fundamental importance of understanding both the characteristics of the data and the properties of the algorithms that process it~\citep{wolpert1997no}. 
Recent research has made significant progress in uncovering key properties of deep learning models, including the edge of stability phenomenon~\citep{wu2018sgd,cohen2021gradient}, frequency principle~\citep{xu2019frequency,xu2022overview}, attention patterns~\citep{elhage2021mathematical,olsson2022context,pmlr-v119-bhojanapalli20a}, parameter distribution properties~\citep{kovaleva-etal-2021-bert,dar2023analyzing}, condensation phenomenon~\citep{luo2021phase,xu2025overview},   and embedding structure~\citep{cai2021isotropy}. 
There has also been some investigation into data characteristics—such as the power-law decay of correlations between elements (like pixels) as a function of their distance~\citep{ruderman1994statistics}. However, there is a significant gap in understanding how these two fundamental aspects—algorithmic properties and data characteristics—interact to produce the remarkable performance.

The embedding space, which acts as the encoder of the language tokens, therefore provides an ideal entry point for investigating how algorithmic properties and data characteristics interact. Ideally, the structure of embeddings should reflect the semantic relationships among tokens. A concrete example involves digits such as $0, 1, 2, \dots$, which possess a natural ordering. Their embedding vectors accordingly display an ordered structure consistent with this numerical sequence, reflecting basic reasoning capabilities in mathematical tasks~\citep{mikolov2013linguistic,ethayarajh2019towards,zhang2024initial,yao2025an}.
However, the cause of the consistency between embedding structures and semantic structures is still an open question, and the driving factors of the embedding structure are still not well characterized.

In this work, we identify a set of probability signatures that encode the semantic information in the data into the structure of the embedding space of language models. Such probability signatures are constructed based on label distribution, input distribution, input/output co-occurrence distribution, etc, that systematically capture inherent token-level relationships and reflect semantic structures. This result is achieved via utilizing an embedding-based model with gradient flow analysis of embedding vectors and unembedding vectors for well-designed variable-controlled experiments. We further extend our findings to LLMs with realistic corpora, such as the Qwen2.5 architecture~\citep{qwen2.5} and subsets of the Pile dataset~\citep{gao2020pile,biderman2022datasheet}. 
The analysis approach with the controlled experiments offers a promising methodology to uncover more and more probability signatures that can bridge the data semantics and embedding structure in language models.

\section{Related Work}

\paragraph{Parameter analysis in LLMs} Investigating the underlying parameter properties in LLMs is crucial for understanding the foundation of models. Some works focus on the specific modules in models. \cite{elhage2021mathematical,olsson2022context} uncover mechanisms such as induction heads from the attention module. \cite{pmlr-v119-bhojanapalli20a} reveals the rank-collapse phenomenon of the attention matrix. \cite{geva-etal-2021-transformer,geva-etal-2022-transformer,dai-etal-2022-knowledge} investigates the characteristics and functions of the FFN in LLMs. Additionally, analysis of a single neuron has been widely employed in mechanism interpretation, particularly in circuits analysis~\cite{hanna2023how,wang2023interpretability,hanna2024have,wang2025towards}, sparse autoencoders (SAE)~\cite{huben2024sparse,bricken2023monosemanticity}, transcoders~\cite{NEURIPS2024_2b8f4db0}, and cross-layer transcoders (CLT)~\cite{ameisen2025circuit}. There are also some studies investigating the global properties of all parameters. \cite{dar2023analyzing,katz-etal-2024-backward} introduce a framework for interpreting all parameters of Transformer models by projecting them into the embedding space. \cite{kovaleva-etal-2021-bert,super-weight} provide an analysis of the parameter distribution, demonstrating the significance of these outliers. In this work, we will focus on the embedding space, explaining the formation of its structure from both experimental and theoretical perspectives.

\paragraph{Embedding structure and representation learning} Since the introduction of static word embeddings by \cite{mikolov2013efficientestimationwordrepresentations,pennington-etal-2014-glove} and the adoption of contextualized embeddings~\citep{devlin2019bertpretrainingdeepbidirectional,peters-etal-2018-deep}, significant attention has been devoted to analyzing embedding properties. \cite{gao2018representation,ethayarajh-2019-contextual,timkey-van-schijndel-2021-bark} explore the anisotropy of embedding space, while \cite{cai2021isotropy} show that embeddings exhibit isotropy within clusters.  \cite{NEURIPS2022_grokking} offers insights into grokking by emphasizing the role of well-organized embedding structures. \cite{zhang2024initial} establishes a connection between embedding structure and model generalization, and \cite{yao2025an} provides an analysis of this relationship. In contrast to prior work, our study focuses on the connection between embedding structure and data properties, offering a novel insight for understanding how embeddings are organized.

\section{Embedding-based Model}
We first explain the basic notation of embedding space. Given a vocabulary $\fV\subset \sN^{+}$ with size $d_{\rm vob}$. We denote a trainable matrix $\vW^E\in\sR^{d\times d_{\rm vob}}$ as the embedding matrix for $\fV$, where $d$ is the hidden dimension. For any $x\in\fV$, we denote $\ve_x\in\sR^{d_{\rm vob}}$ as its one-hot vector and $\vW^E_x:=\vW^E\ve_x$ as the embedding vector of $x$, which is intuitively the $x$-th column of $\vW^E$. For a sequence $\vX=\left[x_1,x_2,\cdots,x_L\right]\subset\fV$, we define its one-hot representation as $\ve_{\vX}:=\left[\ve_{x_1},\ve_{x_2},\cdots,\ve_{x_L}\right]\in\sR^{d_{\rm vob}\times L}$ and $\vW^E_{\vX}:=\left[\vW^E_{x_1},\vW^E_{x_2},\cdots,\vW^E_{x_L}\right]\in\sR^{d\times L}$ as the embedding sequence of $\vX$. Similarly, $\vW^U\in\sR^{d_{\rm vob}\times d}$ represents the unembedding matrix and $\vW^U_{\nu}:=\vW^{U,T}\ve_{\nu}$ (the $\nu$-th row of $\vW^U$) means the unembedding vector for any $\nu\in\fV$. 

We denote the models functioning on the embedding of the input sequence as embedding-based models. We provide the following formulation:
\begin{definition}
Given a vocabulary $\fV\subset \sN^{+}$ with size $d_{\rm vob}$ and a sequence $\vX\in\fV^{L}$ with length $L$. An embedding-based model $F$ taking $\vX$ as input could be formulated as  % \tilde{F} 换成G
\begin{equation*}
    F\left(\vX\right)=\vW^{U}G\left(\vW^{E}_{\vX}\right),
\end{equation*}
where $G$ means the mapping in the hidden space. 

\end{definition}
Embedding-based models have been widely applied in various domains, particularly in NLP. Investigating the characteristics of the embedding space via the gradient flow dynamics is essential for understanding the embedding structures. Given a dataset $\left\{\left(\vX^i, y^i\right)\right\}_{i=1}^N$, we utilize cross-entropy as the loss function:
\begin{equation*}
    \ell^i = -\log \text{Softmax}\left(F\left(\vX^i\right)\right)_{y^i}
     = -\log \frac{\exp F\left(\vX^i\right)_{y^i}}{\sum_{j=1}^{d_{\rm vob}}\exp F\left(\vX^i\right)_{j}}.
\end{equation*}
Denote that $\vp^i = {\rm Softmax}\left(F\left(\vX^i\right)\right)$. Let $\odot$ represent the Hadamard (element-wise) product and $T$ mean the transpose, we obtain the following results.

\begin{prop}\label{prop:dynamics_emb}
    Given an embedding-based model $F$ with an embedding matrix $\vW^{E}$. For any token $x\in\fV$, the gradient flow of $\vW^{E}_x$ (the $x$-th column of $\vW^E$) can be formulated as:
    \begin{equation*}
        \frac{d \vW^{E}_x}{dt} =\sum_{\nu\in\fV}\frac{r_{x,\nu}}{N_{x,\nu}}\left(\vW^{{U}, T}\ve_{\nu}\right)\odot\sum_{i=1}^{N_{x,\nu}}G^{(1)}\left(\vW^E_{\vX_{\left(x,\nu\right)}^i}\right) -\frac{r^{\rm in}_x}{N^{\rm in}_x}\sum_{i=1}^{N^{\rm in}_x}G^{(1)}\left(\vW^E_{\vX^i_x}\right)\odot\left(\vW^{{U}, T}\vp_{x}^i\right),
    \end{equation*} 
    where $N^{\rm in}_x, N_{x,\nu}$ denotes the count of sequences containing $x$ and the count of sequences containing $x$ with label $\nu$, $r^{\rm in}_x=\frac{N^{\rm in}_x}{N},r_{x,\nu}=\frac{N_{s,\nu}}{N}$, $G^{(1)}$ represents the derivative of $G$ with respect to $\vW^{E}_x$. $\vX^i_{x}$ is the $i$-th sequence containing $x$ and $\vX^i_{\left(x,\nu\right)}$ denotes the $i$-th sequence containing $x$ with label $\nu$. $\vp_{x}^i$ means the output probability of the model over $\vX^i_{x}$. 
\end{prop}

\begin{prop}\label{prop:dynamics_unemb}
    Given an embedding-based model $F$ with an unembedding matrix $\vW^{U}$. For any token $\nu\in\fV$, the gradient flow of $\vW^{U}_{\nu}$ (the $\nu$-th row of $\vW^U$) can be written as
    \begin{equation*}
        \frac{d \vW^{U}_{\nu}}{dt} = \frac{r^{\rm out}_{\nu}}{N^{\rm out}_{\nu}}\sum_{i=1}^{N^{\rm out}_{\nu}}\left[G\left (\vW^{E}_{\vX_{\left(\cdot,\nu\right)}^i}\right)\right]^T-\frac{1}{N}\sum_{i=1}^N\vp^{i,\nu} \left[G\left (\vW^{E}_{\vX^i}\right)\right]^T,
    \end{equation*} 
    where $N^{\rm out}_{\nu}$ denotes the count of sequences with label $\nu$ and $r^{\rm out}_{\nu_j}=\frac{N^{\rm out}_{\nu}}{N}$. $\vX_{\left(\cdot,\nu\right)}^i$ means the $i$-th sample which takes $\nu$ as the label. $\vp^{i,\nu}$ means the $\nu$-th element of $\vp^i$. 
\end{prop}

In this work, we employ three embedding-based architectures:
\begin{itemize}
    \item Linear model. $F_{\rm lin}\left(\vX\right) = \vW^{U}\sum_{x\in \vX}\vW^{E}_{x}$.
    \item Feedforward network. $F_{\rm ffn}\left(\vX\right) = \vW^{U}\sigma\left(\sum_{x\in \vX}\vW^{E}_{x}\right)$, where $\sigma$ denotes the element-wise nonlinear activation.
    \item Transformer-based architecture. We employ the Qwen2.5 architecture in Section~\ref{sec:language_model} and the Llama 2 architecture~\citep{llama2} in Appendix~\ref{app:llama}.
\end{itemize}

\section{Probability Signature}\label{sec:distribution}
In the field of deep learning, the data characteristics play a critically important role in both the training dynamics and the final performance of the model. Given a training dataset $\left\{\left(\vX^i, y^i\right)\right\}_{i=1}^N$ sampled from distribution $\pi$, we define the following \emph{probability signatures}:
\begin{equation*}\label{eq:distribution}
\begin{aligned}
    &\vphi_x^{y} = \sum_{\nu\in \fV} \mathbb{P}_{\pi}\left(y=\nu\mid x\in \vX\right)\ve_{\nu},\quad \vphi_x^{\vX} = \sum_{x^{\prime}\in \fV}\mathbb{P}_{\pi}\left(x^{\prime}\in\vX\mid x\in\vX\right)\ve_{x^{\prime}},\\
    &\vphi_x^{\vX\mid y} = \sum_{\nu\in\fV}\ve_{\nu}\left(\sum_{x^{\prime}\in\fV}\mathbb{P}\left(x^{\prime}\in\vX\mid x\in\vX, y=\nu\right)\ve_{x^{\prime}}\right)^T,\quad \vvarphi_{\nu}^{\vX}=\sum_{x\in\fV}\mathbb{P}_{\pi}\left(x\in\vX\mid y=\nu\right)\ve_{x}.
\end{aligned}
\end{equation*}

$\vphi_x^{y}$ denotes the distribution of the label given that the input sequence $\vX$ contains element $x$. It captures the association between a specific input element $x$ and its label.
$\vphi_x^{\vX}$ represents the probability that, given $\vX$ contains $x$, it also contains another input element $x^{\prime}$. It characterizes the co-occurrence relationship between different elements.
$\vphi_x^{\vX\mid y}$ indicates the probability that, when $\vX$ contains $x$ and the label is fixed, the sequence additionally includes $x^{\prime}$. It further delineates the relationship among different elements in the sequence under a specified label condition.
$\vvarphi_{\nu}^{\vX}$ denotes the probability distribution of the element $x$ contained in $\vX$ conditional with the label being $\nu$. It reflects the dependency between input elements and the label from a reverse perspective.

\section{Addition Task}
The addition task has become an important benchmark for studying the characteristics of language models. Many studies have found that the digits' embeddings exhibit an ordered structure consistent with the natural sequence of numbers~\citep{zhang2024initial,yao2025an}. To demonstrate the significant influence of probability signatures on the model's embedding space, we design three types of composite addition tasks to perform \emph{variable-controlled experiments}. Assuming all tokens belong to positive integers, and we denote an anchor set by $\fA$, whose elements represent different addition operations, i.e., anchor $\alpha_1$ means addition with $\alpha_1$. Given a input sequence $\vX=\left[z,\alpha_1,\alpha_2\right]$, which means the composite function $\left(\alpha_1,\alpha_2\right)$ on the key $z$, we define the following tasks:  
\begin{itemize}
    \item Addition task. $f_{\rm add}\left(\vX\right)=z+\alpha_1+\alpha_2,\quad \alpha_1,\alpha_2\in \fA$. For each anchor pair $\left(\alpha_1,\alpha_2\right)$, $z$ is sampled from the same set $\fZ$ with $\fZ\cap\fA=\emptyset$. In $f_{\rm add}$, $\vphi_{\alpha}^{ y}$ are distinct with varying anchor $\alpha$ while $\vphi_{\alpha}^{\vX}$ are identical across anchors.
    \item Addition task with the same value domain. $\tilde{f}_{\rm add}\left(\vX\right)=z+\alpha_1+\alpha_2,\quad \alpha_1,\alpha_2\in \fA$. For anchor pair $\left(\alpha_1,\alpha_2\right)$, $z\in \fZ_{\left(\alpha_1,\alpha_2\right)}=\fY-\alpha_1-\alpha_2$ where $\fY$ denotes the label domain, which is identical for all anchor pairs. In $\tilde{f}_{\rm add}$, $\vphi_{\alpha}^{\vX}$ are distinct with varying anchor $\alpha$ while $\vphi_{\alpha}^{y}$ are identical for all $\alpha\in\fA$.
    \item Module addition. $f_{\rm mod}\left(\vX\right)=\min \fZ + \left(z+\alpha_1+\alpha_2\mod \mid\fZ\mid\right),\quad \alpha_1,\alpha_2\in \fA$ and $z\in\fZ$. Both $\vphi_{\alpha}^{\vX}$ and $\vphi_{\alpha}^{y}$ are identical with different anchors, while $\vphi_\alpha^{\vX\mid y}$ are distinct.
\end{itemize}

In this work, we set $\mathcal{A} = \left\{11, 12, \cdots, 20\right\}$ and $\fY=\mathcal{Z} = \left\{101, 102, \cdots, 140\right\}$. Figure~\ref{fig:addition_acc} A displays the probability signature, which is distinct across the $\alpha$ in each task, revealing that the difference among $\alpha$ lies in the global horizontal shift. The detailed formulations are provided in Appendix~\ref{app:addition_distribution}. 
\begin{figure}[htpb]
    \centering
    \includegraphics[width=1\linewidth]{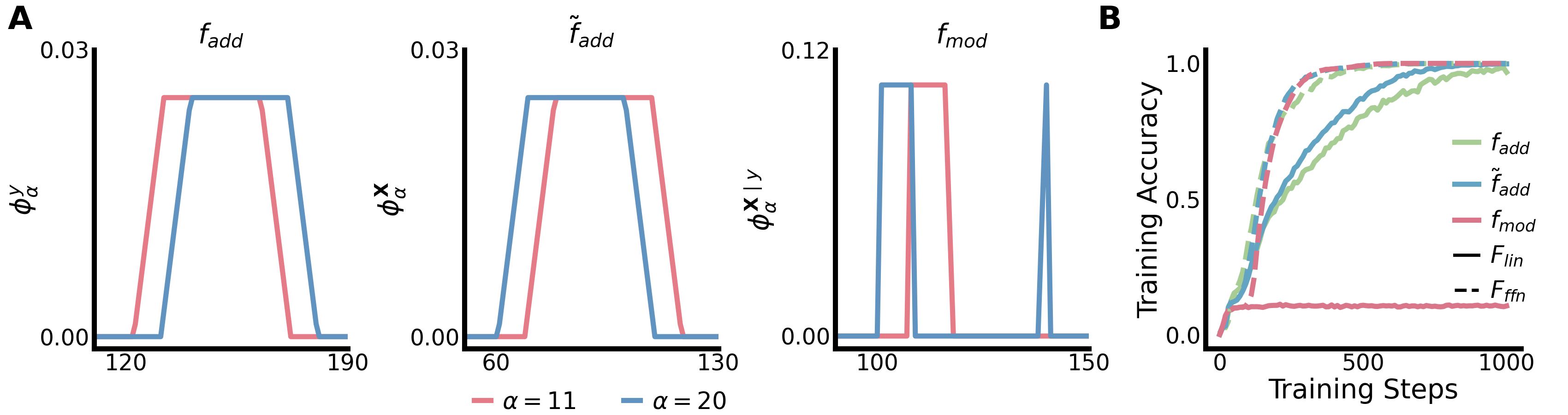}
    \caption{A: Probability signature which is distinct with varying $\alpha$ in each task ($f_{\rm add}\to \vphi_\alpha^y,\tilde{f}_{\rm add}\to\vphi_{\alpha}^{\vX},f_{\rm mod}\to \vphi_{\alpha}^{\vX\mid y}$), $\alpha=11$ (red) and 20 (blue). In $\vphi_{\alpha}^{\vX\mid y}$, it's displayed with $y=160$. B: Training accuracy of the $F_{\rm lin}$ and $F_{\rm ffn}$ on the three addition tasks.}
    \label{fig:addition_acc}
\end{figure}

For notation convenience, we denote that $\vW^E_{\fA}=\left[\vW^{E}_{\alpha}\right]_{\alpha\in\fA}, \vphi_{\fA}=\left[\vphi_{\alpha}\right]_{\alpha\in\fA}$ for $\vphi_{\alpha}=\vphi^{y}_{\alpha},\vphi^{\vX}_{\alpha},\vphi^{\vX\mid y}_{\alpha}$, and $\cos\left(\vW^E_{\fA}\right):=\left[\cos\left(\vW^E_{\alpha},\vW^E_{\alpha^\prime}\right)\right]_{\alpha,\alpha^{\prime}\in\fA},\cos\left(\vphi_{\fA}\right):=\left[\cos\left(\vphi_{\alpha},\vphi_{\alpha^\prime}\right)\right]_{\alpha,\alpha^{\prime}\in\fA}$. Similarly, $\cos\left(\vW^U_{\fV}\right):=\left[\cos\left(\vW^U_{\nu},\vW^U_{\nu^\prime}\right)\right]_{\nu,\nu^{\prime}\in\fV}$ and  $\cos\left(\vvarphi_{\fV}^{\vX}\right):=\left[\cos\left(\vvarphi_{\nu},\vvarphi_{\nu^\prime}\right)\right]_{\nu,\nu^{\prime}\in\fV}$.
\subsection{Results}
We train these addition tasks using the $F_{\mathrm{lin}}$ and $F_{\mathrm{ffn}}$ with $d=200$. Inspired by the work of~\cite{luo2021phase,xu2025overview}, we initialize the model parameters by $\vW_{i,j}\sim\fN\left(0,d^{-0.8}\right)$, indicating a small initialization scale. The complete training configurations are provided in Appendix~\ref{app:exp_setup}.
Figure~\ref{fig:addition_acc} B shows the training accuracy of $F_{\mathrm{lin}}$ and $F_{\mathrm{ffn}}$ on the three addition tasks. The results reveal that both $f_{\mathrm{add}}$ and $\tilde{f}_{\mathrm{add}}$ are learned well by the linear model, whereas $f_{\mathrm{mod}}$ requires the nonlinear model to achieve an effective fit.

\subsection{Embedding Matrix}
In the addition tasks, the anchors exhibit a strict ordering due to the numerical sequence. This provides an ideal setting for the embedding space to develop a corresponding ordered relationship. To formally quantify the formation of the ordered structure, we define the following metric:
\begin{equation*} \label{eq:relation_order}
    R_{\rm order}\left(\vW^{E}_{\fA}\right)={\rm Corr}\left(\cos\left(\vW^E_\fA\right),\left\{\mid \alpha-\alpha^{\prime}\mid\right\}_{\alpha,\alpha^{\prime}\in\fA}\right).
\end{equation*}
 $R_{\rm order}\left(\vW^{E}_{\fA}\right)$ reflects the relationship between embedding similarity and anchor difference. A strong negative $R_{\rm order}\left(\vW^{E}_{\fA}\right)$ (approximately $-1$) indicates that the similarity decreases systematically with increasing anchor difference, confirming the presence of a hierarchical organization in the anchor embeddings. Figure~\ref{fig:addition_embedidng} A represents the distribution of $\cos\left(\vW^E_{\fA}\right)$ for the three tasks with $F_{\rm lin}$ and $F_{\rm ffn}$, respectively, and Figure~\ref{fig:addition_embedidng} B depicts the corresponding evolution of $R_{\rm order}\left(\vW^{E}_{\fA}\right)$. In the case of $f_{\rm add}$, anchor embeddings quickly form an ordered structure, where the cosine similarity gets smaller as the anchor distance gets larger. 
 For the task $\tilde{f}_{\rm add}$, the anchor embeddings also develop a similar hierarchical structure. However, its construction requires more steps, indicating that the driving factors of the structure in $f_{\rm add}$ and $\tilde{f}_{\rm add}$ are different. In $f_{\rm mod}$, although the linear model fails to learn it effectively, the anchor embeddings still undergo noticeable changes from the initial stage. Specifically, all embedding vectors become aligned in nearly the same direction. Furthermore, the anchor embeddings of $f_{\rm mod}$ in $F_{\rm ffn}$ construct an ordered structure with more steps, suggesting that the activation provides another factor in deriving such an embedding structure. 

\begin{figure}[htpb]
    \centering
    \includegraphics[width=1\linewidth]{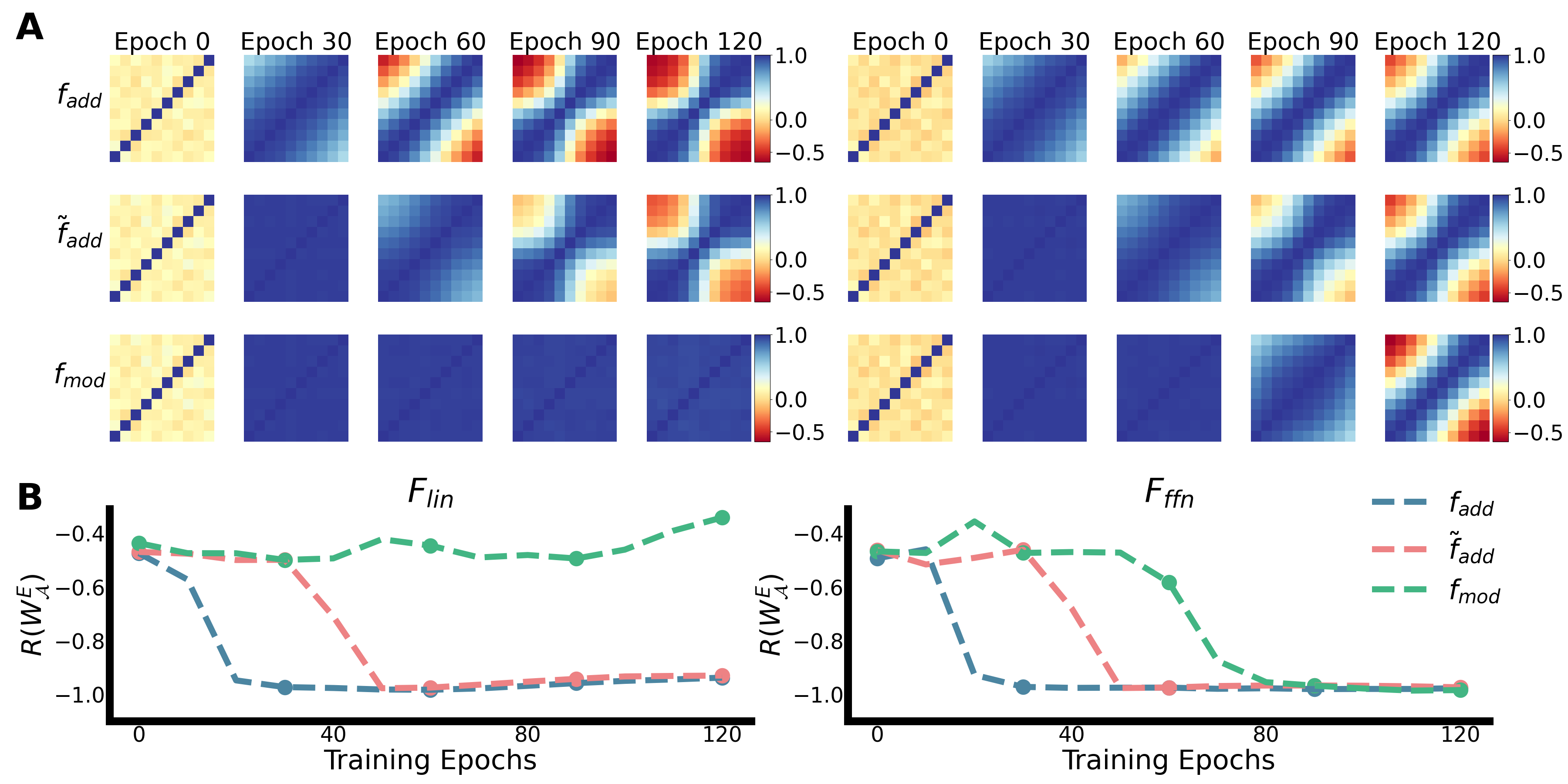}
    \caption{A: The heatmap of $\cos\left(\vW^E_{\fA}\right)$ in $F_{\rm lin}$ (left)  and $F_{\rm ffn}$ (right) during the training process. Each row corresponds to $f_{\rm add}$, $\tilde{f}_{\rm add}$, and $f_{\rm mod}$, respectively. B: Dynamics of $R_{\rm order}\left(\vW^{E}_{\fA}\right)$ in $F_{\rm lin}$ (left) and $F_{\rm ffn}$ (right). Line colors represent task types. }
    \label{fig:addition_embedidng}
\end{figure}

We derive the reasons for the embedding structure in each task from the gradient flow. With Proposition~\ref{prop:dynamics_emb}, we obtain the following approximation: 
\begin{corollary} [Embedding of Linear Model]\label{cor:emb_linear}
    Let $N\rightarrow\infty$, $\pi$ denotes the data distribution over the training set. The gradient flow of $\vW^E_\alpha$ in $F_{\rm lin}$ can be approximated by
    \begin{equation}\label{eq:linear_emb}
    \frac{d\vW^{E}_{\alpha}}{dt}=\vW^{{U},T} r_x^{\rm in}\left(\vphi_\alpha^{y}-\frac{1}{d_{\rm vob}}\vW^{U}\vW^{E}\vphi^{\vX}_{\alpha}+\veta\right),
    \end{equation}
where $\veta$ denotes the data-independent and higher-order terms.
\end{corollary}
\begin{remark}
Deep learning methods fundamentally comprise three essential components: the model, the data, and the optimization algorithm. Corollary~\ref{cor:emb_linear} clearly illustrates how these three elements are coupled (model: $F_{\rm lin}$; data: probability signatures; optimization algorithm: the gradient flow) and influence the embedding space, providing crucial insights for understanding and investigating their joint impact on the performance of deep learning approaches. 
\end{remark}

Corollary~\ref{cor:emb_linear} indicates that the dynamics of $\vW^E_{\alpha}$ in $F_{\rm lin}$ are primarily impacted by the probability signatures $\vphi_\alpha^y$ and $\vphi_\alpha^{\vX}$, demonstrating the connection between data distribution and the embedding space. As we mentioned, the $\vphi_\alpha^{y}$ is distinct for different $\alpha$, while the $\vphi_{\alpha}^{\vX}$ is identical for all $\alpha$ in $f_{\rm add}$; the opposite holds for $\tilde{f}_{\rm add}$. Figure~\ref{fig:cosine_esitimation_addition} A depicts $\cos\left(\vphi_{\fA}^y\right)$ (left) and $\cos\left(\vphi_{\fA}^{\vX}\right)$ (middle column) in $f_{\rm add}$ (top) and $\tilde{f}_{\rm add}$ (middle row), revealing that $\vphi_{\alpha}^{y}$ and $\vphi_{\alpha}^{\vX}$ are significant for the formation of the hierarchy embedding structure in $f_{\rm add}$ and $\tilde{f}_{\rm add}$, respectively. Furthermore, \eqref{eq:linear_emb} indicates that $\vphi_{\alpha}^y$ acts as a leading term and the effect of $\vphi_\alpha^{\vX}$ is weaker in the early training process since it times a small magnitude term $\frac{1}{d_{\rm vob}}\vW^{U}\vW^{E}$. This results in the formation speed of the structure in $\tilde{f}_{\rm add}$ being slower than $f_{\rm add}$, which is consistent with the phenomenon in Figure~\ref{fig:addition_embedidng}.

\begin{figure}[ht]
    \centering
        \includegraphics[width=1\linewidth]{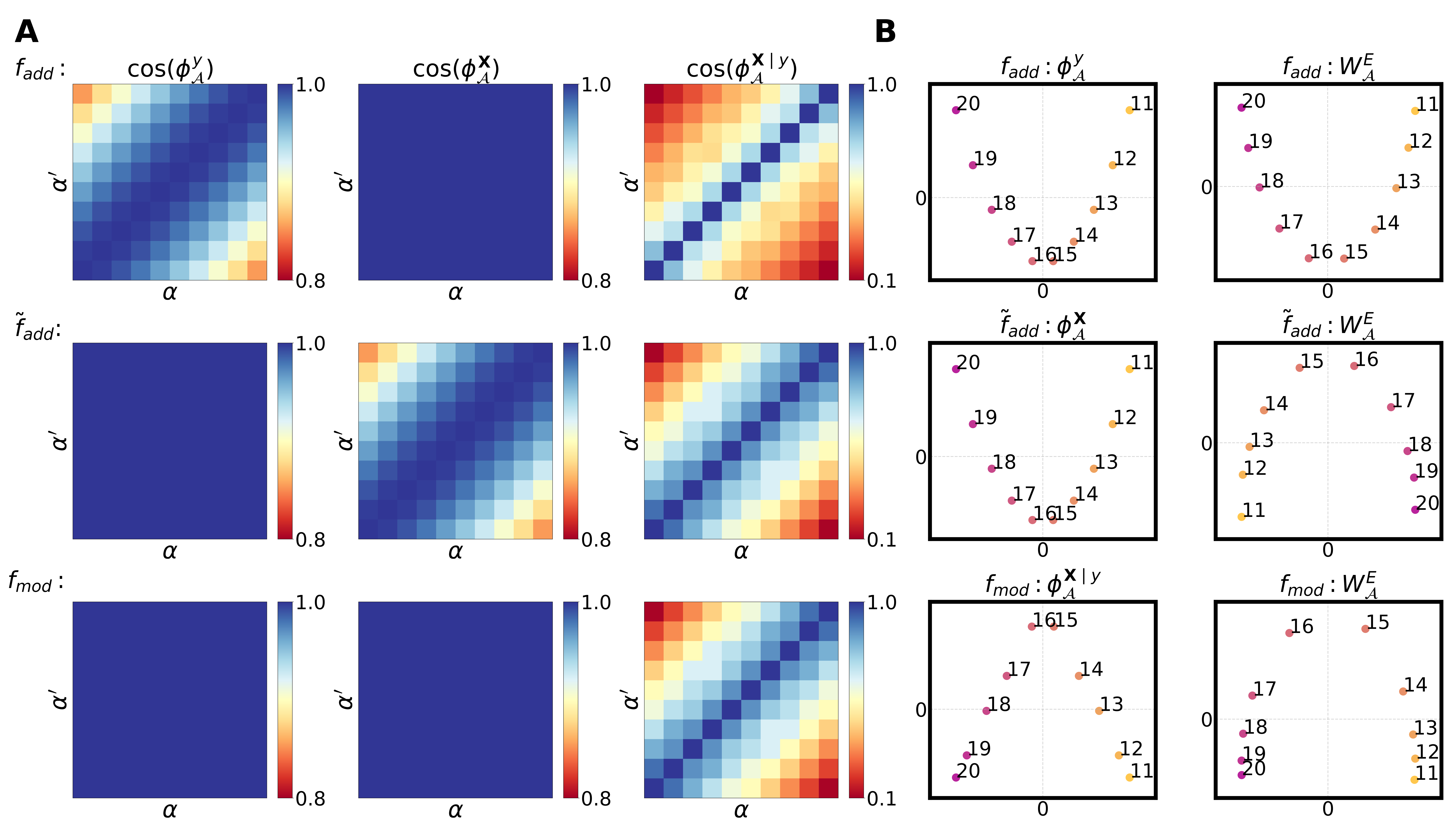}
    
    \caption{A: Cosine similarity among different anchor $\alpha$ of $\vphi_{\alpha}^{y},\vphi_{\alpha}^{\vX},\vphi_{\alpha}^{\vX\mid y}$ (see \eqref{eq:distribution}) in each task. B: The PCA projection of the key factors ($f_{\rm add}\to \vphi_\alpha^y,\tilde{f}_{\rm add}\to\vphi_{\alpha}^{\vX},f_{\rm mod}\to \vphi_{\alpha}^{\vX\mid y}$) and the embedding vectors in different tasks ($F_{\rm ffn}$, 120 epoch). }
    \label{fig:cosine_esitimation_addition}
\end{figure}
In task $f_{\rm mod}$, $\vphi_{\alpha}^{y}$ and $\vphi_{\alpha}^{\vX}$ are both identical across different anchors $\alpha$. Figure~\ref{fig:cosine_esitimation_addition} A (bottom) indicates that $\cos\left(\vphi_{\fA}^{y}\right)$ and $\cos\left(\vphi_{\fA}^{\vX}\right)$ are $1$ for all anchor pairs, which leads these embedding vectors to converge to almost the same direction, consistent with the observation in $F_{\rm lin}$. To identify the key factors that contribute to the formation of the ordered embedding structure for $f_{\rm mod}$ in $F_{\rm ffn}$, we perform a similar analysis of its gradient flow and obtain the following result.
\begin{corollary}[Embedding of FFN]\label{cor:emb_nonlinear}
        Let $N\rightarrow\infty$, $\pi$ denotes the data distribution over the training set. The gradient flow of $\vW^E_\alpha$ in $F_{\rm ffn}$ could be approximated by
        \begin{equation}\label{eq:nonlinear_emb}
        \frac{d\vW^{E}_{\alpha}}{dt} =\mathbb{T}\cdot\left(\vphi_{\alpha}^{\vX\mid y}\right)^T  + \veta_{\vphi_{\alpha}^{y}}+ \frac{1}{d_{\rm vob}}\veta_{\vphi_{\alpha}^{\vX}} + \tilde{\veta},
        \end{equation}
where $\mathbb{T}\in\sR^{d\times d_{\rm vob}\times d_{\rm vob}}$, $\mathbb{T}_{:,:,\nu}=r_{\alpha,\nu}{\rm diag}\left(\vW^{U}_\nu\right)\vW^{E}$ for $\nu\in\fV$ and 0 otherwise. $\veta_{\vphi_{\alpha}^{y}}$ and $\veta_{\vphi_{\alpha}^{\vX}}$ denotes the term related with $\vphi_{\alpha}^{y}$ and $\vphi_{\alpha}^{\vX}$, respectively. $\tilde{\veta}$ represents the higher-order term.
\end{corollary}

Corollary~\ref{cor:emb_nonlinear} indicates that the ordered embedding structure of $f_{\rm mod}$ primarily relies on probability signature $\vphi_{\alpha}^{\vX\mid y}$. Figure~\ref{fig:cosine_esitimation_addition} A (right) depicts the $\cos\left(\vphi_{\fA}^{\vX\mid y}\right)$ in tasks, which reveals that $\vphi_{\alpha}^{\vX\mid y}$ in $f_{\rm mod}$ constructs an ordered structure, resulting in the ordered embedding structure in $F_{\rm ffn}$. Furthermore, Figure~\ref{fig:cosine_esitimation_addition} C shows the PCA projection on probability signatures and the embedding space in $F_{\rm ffn}$, revealing a high consistency. This comparison demonstrates the impact of the probability signatures in shaping the embedding space.

\subsection{Unembedding Matrix}
The $i$-th row of the unembedding matrix can also be viewed as the feature for the $i$-th token. As shown in Figure~\ref{fig:addition_unemb} A, a similar ordered structure emerges among the unembedding vectors with the label index in $F_{\rm lin}$ across all tasks. Specifically, $\vW^{U}_{\nu}$ in $f_{\rm mod}$ constructs a ring where the similarity between small $\nu$ and large $\nu^{\prime}$ is also large since $\fZ_{\max} + 1 = \fZ_{\min}$ in $f_{\rm mod}$. 

\begin{figure}[h!]
    \centering
    \includegraphics[width=1\linewidth]{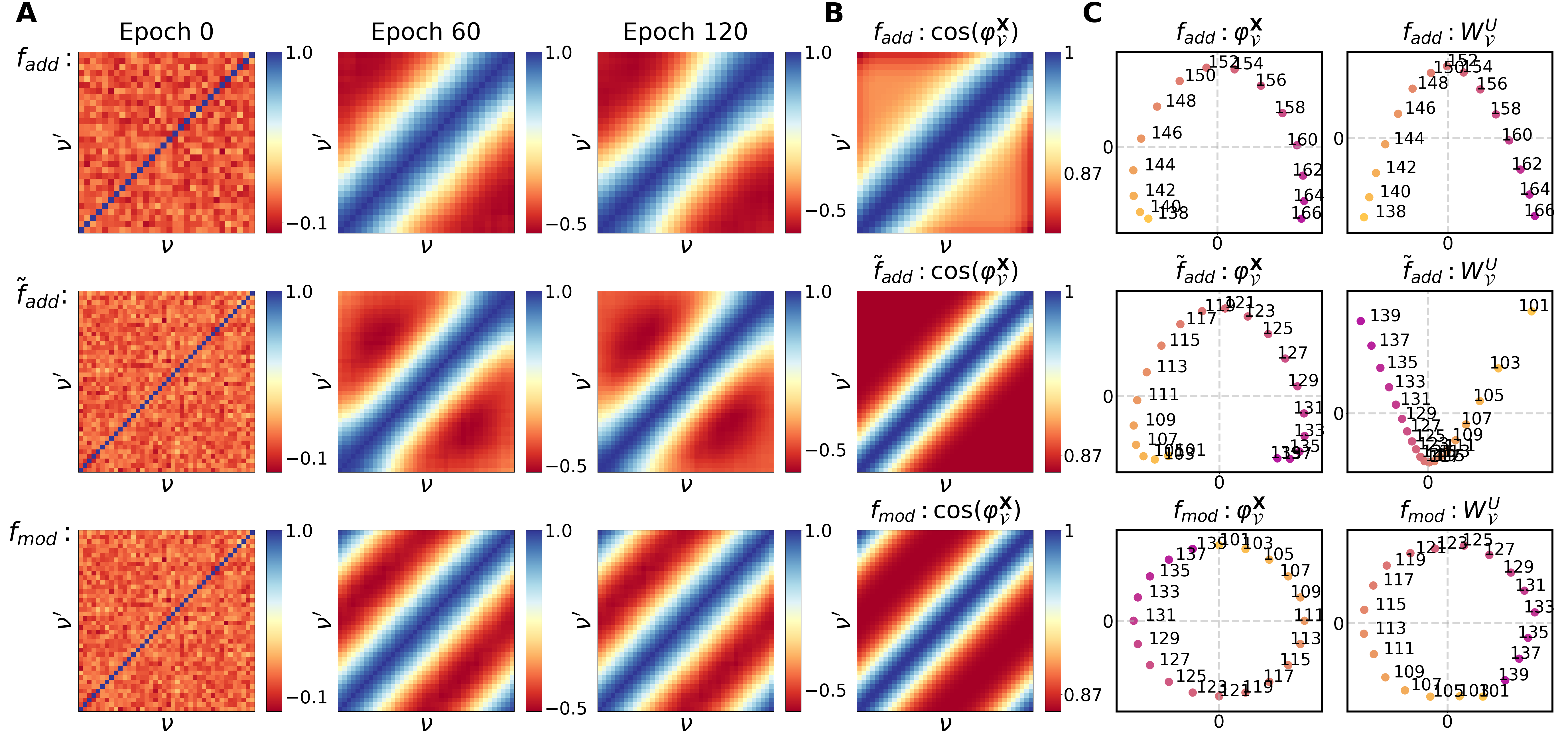}
    \caption{A: The heatmap of the $\cos\left(\vW^U_{\fV}\right)$ with label index in $F_{\rm lin}$ during the training process. B: The heatmap of $\cos\left(\vvarphi_{\fV}^{\vX}\right)$ across different tasks. C: PCA projection of $\vvarphi^{\vX}_{\fV}$ and $\vW^U_{\fV}$ (epoch 120).}
    \label{fig:addition_unemb}
\end{figure}

Similarly, we identify the driving factors of this specific structure by examining the gradient flow of $\vW^{U}$. Since this phenomenon occurs in both $F_{\rm lin}$ and $F_{\rm ffn}$, it suffices to analyze the linear model. Based on Proposition~\ref{prop:dynamics_unemb}, we derive the following result:
\begin{corollary}[Unembedding of Linear Model]\label{cor:unemb_linear}
    Let $N\rightarrow\infty$, $\pi$ denotes the data distribution over the training set. The gradient flow of $\vW^U_\nu$ in $F_{\rm lin}$ could be approximated by
    \begin{equation}\label{eq:unemb}
        \frac{d\vW^{U}_{\nu}}{dt} =L r^{\rm out}_{\nu}\left(\vW^{E}\vvarphi_{\nu}^{\vX}\right)^T + \veta,
    \end{equation}
    where $\veta$ denotes the output term.
\end{corollary}

Corollary~\ref{cor:unemb_linear} illustrates that $\vvarphi_{\nu}^{\vX}$ plays a significant role in shaping the unembedding matrix. Figure~\ref{fig:addition_unemb} B depicts the distribution of $\cos\left(\vvarphi_{\fV}^{\vX}\right)$, which is aligned with the distribution of the $\cos\left(\vW^U_{\fV}\right)$ in Figure~\ref{fig:addition_unemb} A. Furthermore, Figure~\ref{fig:addition_unemb} C compares the PCA projection of $\vvarphi_{\fV}^{\vX}$ and $\vW^U_{\fV}$ in all tasks, revealing a high consistency and validating our analysis.

\section{Language Model}\label{sec:language_model}

\begin{figure}[t!]
    \centering
    \includegraphics[width=1\linewidth]{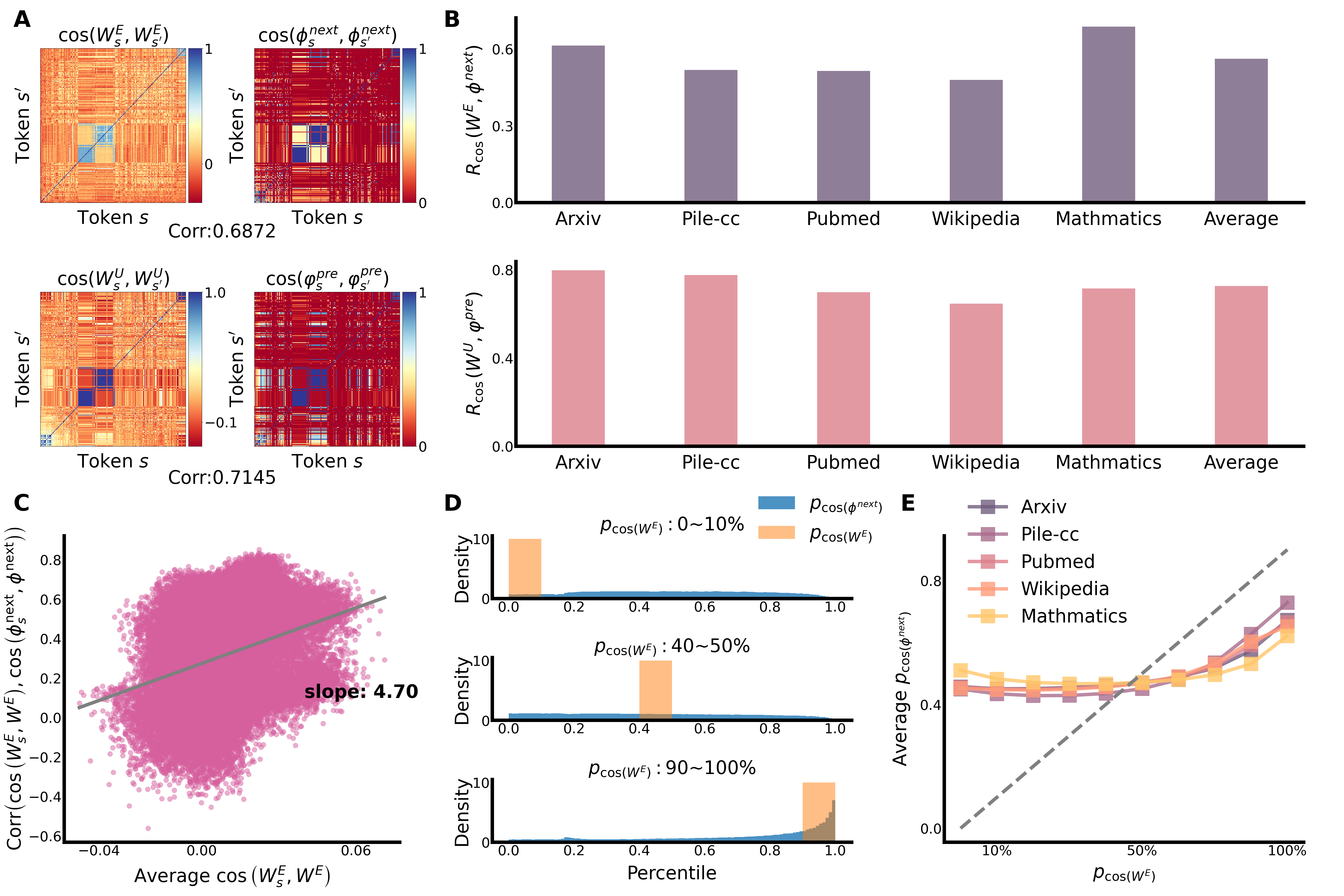}
    \caption{A: Heatmap of the cosine similarity of $\vW^E,\vW^U,\vphi^{\rm next}$ and $\vvarphi^{\rm pre}$. B: $R_{\cos}\left(\vW^E,\vphi^{\rm next}\right)$ (top) and $R_{\rm cos}\left(\vW^U,\vvarphi^{\rm pre}\right)$ (bottom) with different datasets. C: Relation between ${\rm Corr}\left(\cos\left(\vW^E_s,\vW^E\right),\cos\left(\vphi_s^{\rm next},\vphi^{\rm next}\right)\right)$ and the average value of $\cos\left(\vW^E_s,\vW^E\right)$. Each point denotes a token $s$. D: Distribution of $p_{\cos\left(\vphi^{\rm next}\right)}$, conditioned on intervals $0\sim10 \%, 40\sim50 \%$ and $90\sim100 \%$ of the $p_{\cos\left(\vW^E\right)}$. E: Average value of $p_{\cos\left(\vphi^{\rm next}\right)}$ within each interval of $p_{\cos\left(\vW^E\right)}$.}
    \label{fig:qwen_train}
\end{figure}

We have demonstrated the influence of data distribution on the embedding space in the addition tasks. In this section, we explore how to extend this analysis to real-world language models. Most contemporary language models are built upon the Transformer decoder architecture. Assuming the input sequence is denoted as 
$\vX$ with length $L$, we define a language model $F_{\rm lan}$ as:
\begin{align*}
    F_{\rm lan}\left(\vX\right)=\vW^{U}\left(\vW^{E}_{\vX}+\tilde{F}\left(\vX\right)\right)
\end{align*}

Given the training corpus $\left\{\vX^i\right\}_{i=1}^N$, we define the following probability signatures for any $s\in\fV$:
\begin{equation}\label{eq:next_dist}
\begin{aligned} 
    &\vphi_{s}^{\rm next} = \sum_{s^{\prime}\in \fV}\mathbb{P}_{\pi}\left(\cup_{t=1}^{L-1}\left\{X_{t+1}=s^{\prime}\mid X_t=s\right\}\right)\ve_{s^{\prime}},\\ &\vvarphi_s^{\rm pre} = \sum_{s^{\prime}\in\fV}\mathbb{P}_{\pi}\left(\cup_{t=1}^{L-1}\left\{ X_t=s^{\prime}\mid X_{t+1}=s\right\}\right)\ve_{s^{\prime}},
\end{aligned}
\end{equation}
and $\vphi^{\rm next} = \left[\vphi_{s}^{\rm next}\right]_{s\in\fV},\vvarphi^{\rm pre} = \left[\vphi_{s}^{\rm pre}\right]_{s\in\fV}$. We derive the following result:
\begin{corollary}\label{cor:lan}
    Let $N\rightarrow\infty$, $\pi$ denotes the token distribution in the training dataset. The gradient flow of the embedding vector $\vW^{E}_s$ of token $s$ could be fomulated as
    \begin{align*}
        \frac{d\vW^{E}_{s}}{dt} =&r^{\rm in}_s\vW^{{U},T}\vphi_s^{\rm next} + \veta^E.
    \end{align*}
    Furthermore, the gradient flow of the unembedding vector $\vW^U_s$ could be approximated as
    \begin{align*}
    \frac{d\vW^U_s}{dt} = r_{s}^{\rm out}\left(\vW^{E}\vvarphi_s^{\rm pre}\right)^T +\veta^U.
\end{align*}
The $\veta^E$ and $\veta^U$ denote the output probability and the higher-order term.

\end{corollary}

Corollary~\ref{cor:lan} suggests that given any token $s$,  the distributions of its next token and previous token significantly impact its embedding. We trained a group of Qwen2.5 models on different subsets of the Pile. Figure~\ref{fig:qwen_train} A shows these similarity matrices for the dataset Pile-dm-mathematics, where the tokens displayed are those that occur most frequently in the corpus. We define the following correlation coefficient $R_{\cos}\left(\vW^E,\vphi^{\rm next}\right):={\rm Corr}\left(\cos\left(\vW^{E}\right),\cos\left(\vphi^{\rm next}\right)\right)$, and similarly $R_{\rm cos}\left(\vW^U,\vvarphi^{\rm pre}\right)$. Figure~\ref{fig:qwen_train} B depicts both metrics across all subsets, suggesting that probability signatures significantly impact the structure of the embedding space and reflect the relationships among embeddings. Furthermore, we find that the probability signatures reflect the strong connections of embeddings more faithfully. As shown in Figure~\ref{fig:qwen_train} C, the correlation between ${\rm Corr}\left(\cos\left(\vW^E_s,\vW^E\right),\cos\left(\vphi_s^{\rm next},\vphi^{\rm next}\right)\right)$ and $\cos\left(\vW^E_s,\vW^E\right)$ is plotted against for all tokens $s$, demonstrating stronger consistency in high-similarity regions. We define $p_{\cos\left(\vW^E\right)}$ and $p_{\cos\left(\vphi^{\rm next}\right)}$ as the percentile matrix of each elements in $\cos\left(\vW^E\right)$ and $\cos\left(\vphi^{\rm next}\right)$, respectively. Figure~\ref{fig:qwen_train} D displays the distribution of $p_{\cos\left(\vphi^{\rm next}\right)}$, conditioned on different intervals of the $p_{\cos\left(\vW^E\right)}$, and Figure~\ref{fig:qwen_train} E shows the average value of $p_{\cos\left(\vphi^{\rm next}\right)}$ within each interval of $p_{\cos\left(\vW^E\right)}$. It can be observed that the alignment is significantly stronger in the regions with large embedding similarity. In Appendix~\ref{app:qwen}, we provide a detailed method explanation, a specific case of the token group with large similarity, and an analysis with the Llama-2 architecture to validate the generalization of our analysis. 

Since general-purpose pretrained base models are trained on broad corpora, we attempt to directly estimate their embedding structure with a subset of general text. We combine all datasets employed in Figure~\ref{fig:qwen_train} and define $\tilde{\vphi}=\vphi^{\rm next}+\vvarphi^{\rm pre}$ (Since the tied embedding, the detail is provided in Appendix~\ref{app:lan_tech}). We compare the $\cos\left(\tilde{\vphi}\right)$ with $\cos\left(\vW^E\right)$ of Qwen2.5-3B-base. As shown in Figure~\ref{fig:qwen_emb} A, the structure of $\tilde{\vphi}$ could capture the main properties of the pre-trained model’s embedding structure, particularly the presence of sub-blocks with high similarity. Furthermore, we examine the instance for the digits ranging from 0 to 9. Figure~\ref{fig:qwen_emb} B illustrates the $\cos\left(\vW^E\right)$ and $\cos\left(\tilde{\vphi}\right)$ of such digits, both revealing an ordered organization that aligns with their numerical sequence. It should be noted that this estimation may not generalize across all open-source base models, as it is sensitive to both the initialization of the pre-trained model and the true training dataset.
\begin{figure}[htbp]
    \centering
    \includegraphics[width=0.9\linewidth]{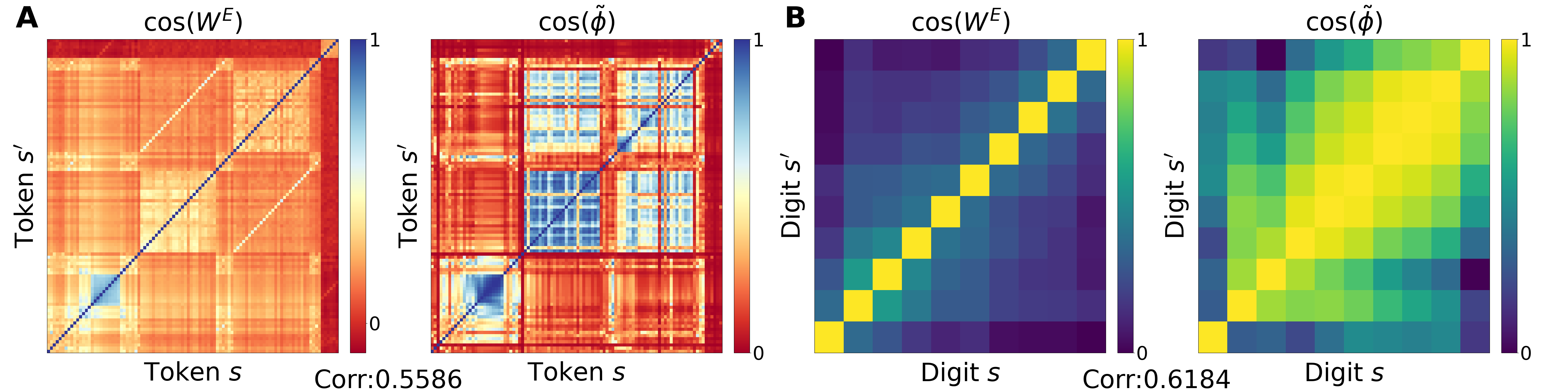}
    \caption{Cosine similarity of $\vW^E$ of the Qwen2.5-3B-base and $\tilde{\vphi}$, respectively, with the frequently-appearing tokens (A) and the digits from $0$ to $9$ (B).}
    \label{fig:qwen_emb}
\end{figure}

\section{Conclusion}
In this work, we investigate the formation of embedding structures in language models. By interpreting the relationship between embedding organization and semantic structure through the lens of data distribution, we propose the probability signatures and design the addition tasks to conduct variable-controlled experiments. Our findings demonstrate that probability signatures play a crucial role in shaping the embedding structure and reflecting underlying semantic relationships. An extended analysis of LLMs further confirms our analysis. This study establishes a bridge between data semantics and embedding space, offering new insights into the understanding of the joint impact of model, data and optimization method. For future work, we plan to extend our theoretical analysis into a comprehensive framework. Besides, we aim to incorporate the self-attention mechanism into our analysis of the LLMs, which is essential for capturing more subtle and complex relationships among embeddings that remain beyond the reach of our current methods.

\bibliography{iclr2026_conference}
\bibliographystyle{iclr2026_conference}

\newpage
\appendix

\section{Experimental Setups}\label{app:exp_setup}

\paragraph{Addition tasks} For each type of addition task, we trained a linear model $F_{\rm lin}$ and a Feedforward network $F_{\rm ffn}$. The hidden size $d=200$, and we employed the ReLU as the activation function. Each dataset contains 50000 data pairs. The training is conducted for 1000 epochs with a batch size of 100. The AdamW optimizer is employed with an initial learning rate of $10^{-5}$.

\paragraph{Language Models} In the analysis of the LLMs, we employ the Qwen2.5 architecture with 12 layers and 12 attention heads in each layer. We set up that the hidden size is 512, and the intermediate size in FFN is 1024. The dimension of the key vectors and value vectors in each head is 64. Similarly, we initialize the parameter by $\vW_{i,j}\sim\fN\left(0,d_{\rm in}^{-1}\right)$ where $d_{\rm in}$ means the input dimension of $\vW$. We select five subsets of Pile, including Pile-arxiv, Pile-dm-mathematics, Pile-cc, Pile-pubmed-central, and Pile-wikipedia-en. The length of each sequence is 2048. The training is conducted for 1 epoch in each experiment, with the AdamW optimizer and a cosine learning rate schedule
 utilized. The initial learning rate is $10^{-4}$.

\section{Addition Task}
\subsection{Probability Signatures in Addition Tasks}\label{app:addition_distribution}
We provide a formulation of the following probability in the three addition tasks. We denote $U\left(\fA\right)$ and $U\left(\fZ\right)$ as the discrete uniform distribution over $\fA$ and $\fZ$, respectively. $A$ and $Z$ are the random variables following $U\left(\fA\right)$ and $U\left(\fZ\right)$. For the task $f_{\rm add}$, we have that
\begin{align*}
    &\mathbb{P}_{\pi}\left(y=\nu\mid \alpha\in \vX\right) =\mathbb{P}_{\pi}\left(A+Z=\nu-\alpha\right),\quad \mathbb{P}_{\pi}\left(z\in\fX\mid \alpha\in \vX\right) =\frac{1}{|\fZ|},\\
    &\mathbb{P}_{\pi}\left(z\in\vX\mid \alpha\in \vX,y=\nu\right) = \mathbb{P}_{\pi}\left(A=\nu-\alpha-z\right)=\frac{1}{|\fA|}\delta_{\nu-\alpha-z\in \fA}, \\
    &\mathbb{P}_{\pi}\left(\alpha^{\prime}\in\vX\mid \alpha\in \vX,y=\nu\right) = \mathbb{P}_{\pi}\left(Z=\nu-\alpha-\alpha^\prime\right)=\frac{1}{|\fZ|}\delta_{\nu-\alpha-\alpha^\prime\in \fZ}, \\
    &\mathbb{P}_{\pi}\left(z\in\vX\mid y=\nu\right)=\mathbb{P}_{\pi}\left(A+A=\nu-z\right),\quad \mathbb{P}_{\pi}\left(\alpha\in\vX\mid y=\nu\right)=\mathbb{P}_{\pi}\left(A+Z=\nu-\alpha\right),
\end{align*}
where $\alpha,\alpha^\prime\in\fA,z\in\fZ$.It's noted that besides the co-occurrence probability $\mathbb{P}_{\pi}\left(z\in\fX\mid \alpha\in \vX\right)$, the value of other ones is dependent on $\alpha$ or $\nu$. Figure~\ref{fig:addition_distribution} (left) displays the distribution of these probabilities, which intuitively reveals the cause of the hierarchy structure in the similarity matrix. Similarly, for $\tilde{f}_{\rm add}$, denote $Y\sim U\left(\fY\right)$ and we have
\begin{align*}
    &\mathbb{P}_{\pi}\left(y=\nu\mid \alpha\in \vX\right) =\frac{1}{|\fY|},\quad \mathbb{P}_{\pi}\left(z\in\fX\mid \alpha\in \vX\right) =\mathbb{P}_{\pi}\left(Y-A=z+\alpha\right),\\
    &\mathbb{P}_{\pi}\left(z\in\vX\mid \alpha\in \vX,y=\nu\right) = \mathbb{P}_{\pi}\left(A=\nu-\alpha-z\right)=\frac{1}{|\fA|}\delta_{\nu-\alpha-z\in \fA}, \\
    &\mathbb{P}_{\pi}\left(\alpha^{\prime}\in\vX\mid \alpha\in \vX,y=\nu\right) = \frac{1}{|\fZ|}, \\
    &\mathbb{P}_{\pi}\left(z\in\vX\mid y=\nu\right)=\mathbb{P}_{\pi}\left(A+A=\nu-z\right),\quad \mathbb{P}_{\pi}\left(\alpha\in\vX\mid y=\nu\right)=\mathbb{P}_{\pi}\left(A+Z=\nu-\alpha\right).
\end{align*}
For $f_{\rm mod}$, we have 
\begin{align*}
    &\mathbb{P}_{\pi}\left(y=\nu\mid \alpha\in \vX\right) =\frac{1}{|\fZ|},\quad \mathbb{P}_{\pi}\left(z\in\fX\mid \alpha\in \vX\right) =\frac{1}{|\fZ|},\\
    &\mathbb{P}_{\pi}\left(z\in\vX\mid \alpha\in \vX,y=\nu\right) =\frac{1}{|\fA|}\delta_{\nu-\min\fZ-\left(\alpha-z\text{  mod}|\fZ|\right)\in \left(A\text{  mod}|\fZ|\right)}, \\
    &\mathbb{P}_{\pi}\left(\alpha^{\prime}\in\vX\mid \alpha\in \vX,y=\nu\right) = \frac{1}{|\fZ|}, \\
    &\mathbb{P}_{\pi}\left(z\in\vX\mid y=\nu\right)=\mathbb{P}_{\pi}\left(\left(A+A\text{  mod} |\fZ|\right)=\nu-\min\fZ-\left(z\text{  mod}|\fZ|\right)\right),\\
    &\mathbb{P}_{\pi}\left(\alpha\in\vX\mid y=\nu\right)=\mathbb{P}_{\pi}\left(\left(A+Z\text{  mod} |\fZ|\right)=\nu-\min\fZ-\left(\alpha\text{  mod}|\fZ|\right)\right).
\end{align*}
Figure~\ref{fig:addition_distribution} depicts all these probability distributions.
\begin{figure}[htbp]
    \centering
    \includegraphics[width=0.9\linewidth]{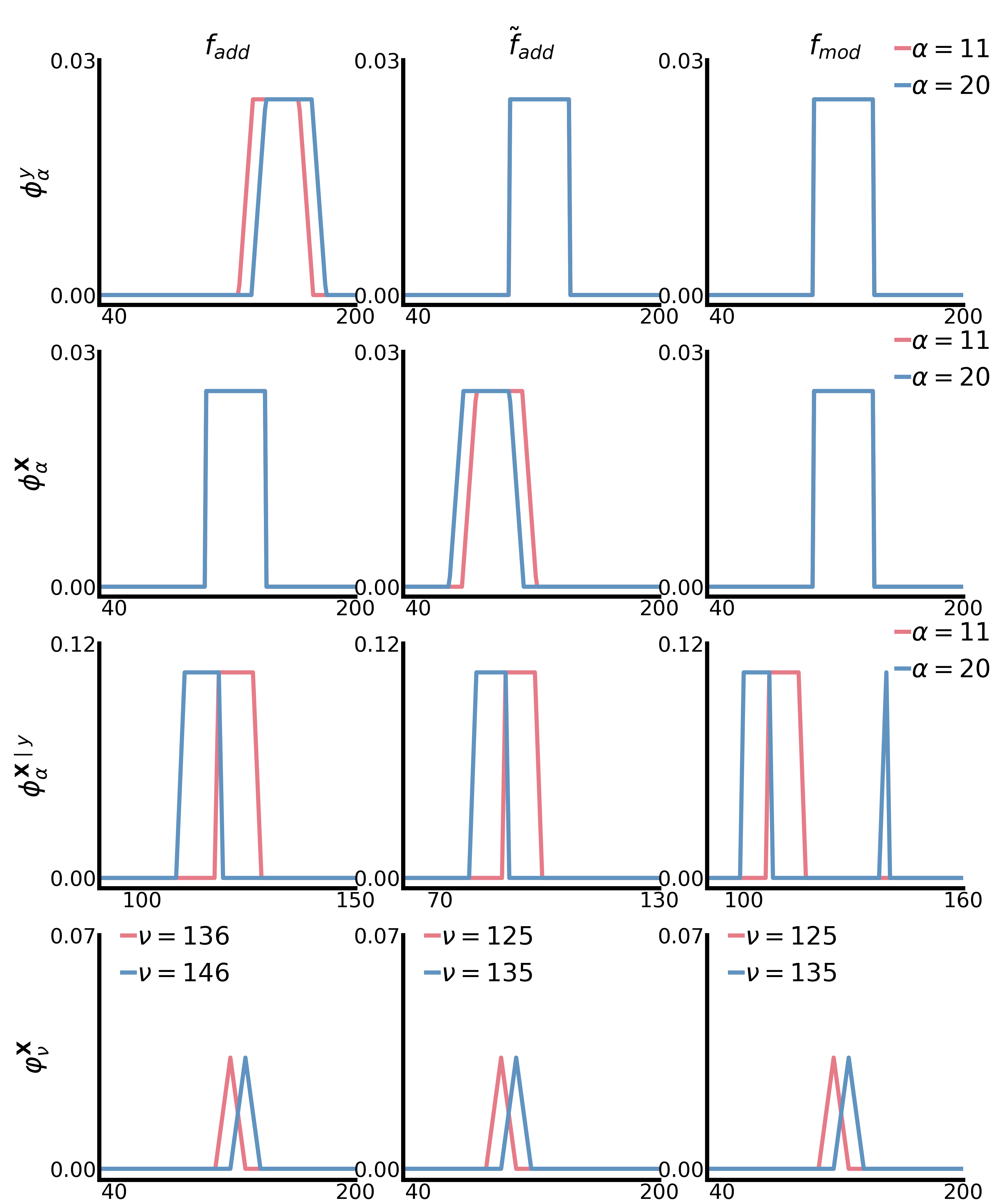}
    \caption{Probability signatures in each task under distinct $\alpha$ and $\nu$. In the distribution of $\vphi_{\alpha}^{\vX\mid y}$, $y=150$ is displayed in $f_{\rm add}$ and $y=120$ in $\tilde{f}_{\rm add}$ and $f_{\rm mod}$, since  $150$ and $120$ are the average label value in each task.}
    \label{fig:addition_distribution}
\end{figure}

\subsection{Embedding matrix in Linear Model}
Figure~\ref{fig:embedding_pca} depicts the PCA projection of the anchor embeddings in $F_{\rm lin}$, revealing that $f_{\rm add}$ and $\tilde{f}_{\rm add}$ both establish an ordered structure while the anchor embeddings in $f_{\rm mod}$ are chaotic.
\begin{figure}[htbp]
    \centering
    \vspace{-10pt}
    \includegraphics[width=0.9\linewidth]{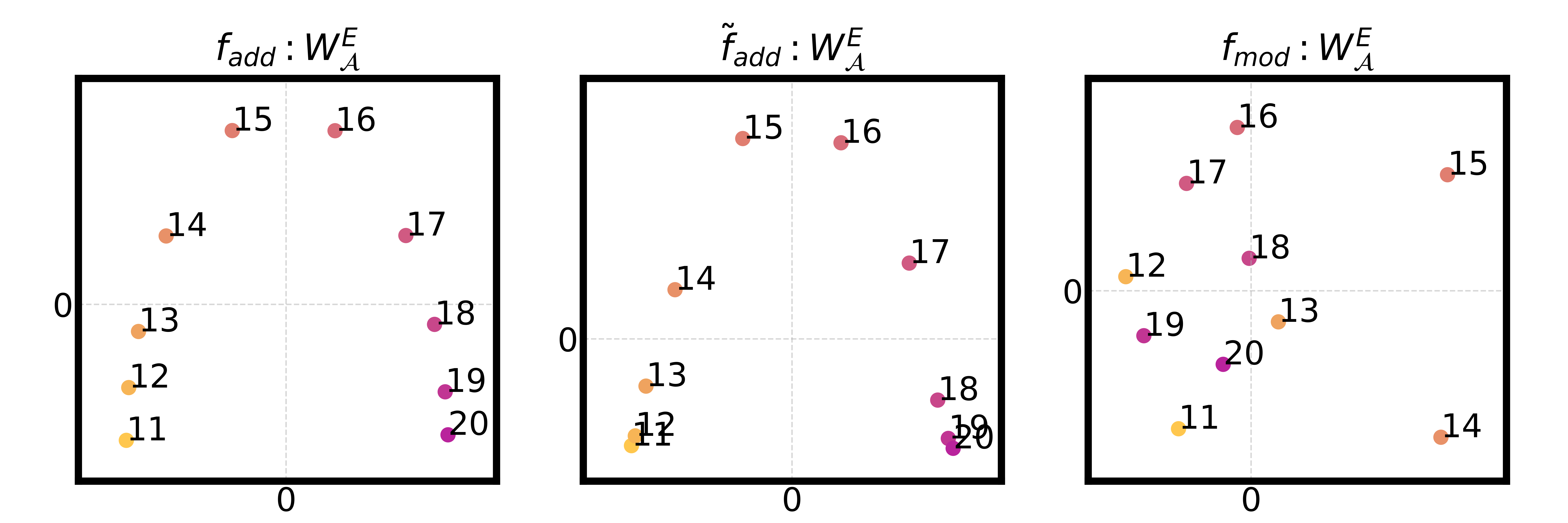}
    \caption{PCA projection of  $\vW^E_{\fA}$ in $F_{\rm lin}$ (epoch 120).}
    \label{fig:embedding_pca}
\end{figure}

\subsection{Umembedding matrix in Feedforward Network}
Figure~\ref{fig:ffn_unembedding} displays the structure of the unembedding matrix in $F_{\rm ffn}$ with the three types of addition tasks. The distribution of $\cos\left(\vW^U_{\nu}\right)$ (A) and the PCA projection (B) jointly reveal that the unembedding vectors of those label tokens establish a hierarchy structure, which is consistent with their natural sequence. 
\begin{figure}[htbp]
    \centering
    \vspace{-10pt}
    \includegraphics[width=0.9\linewidth]{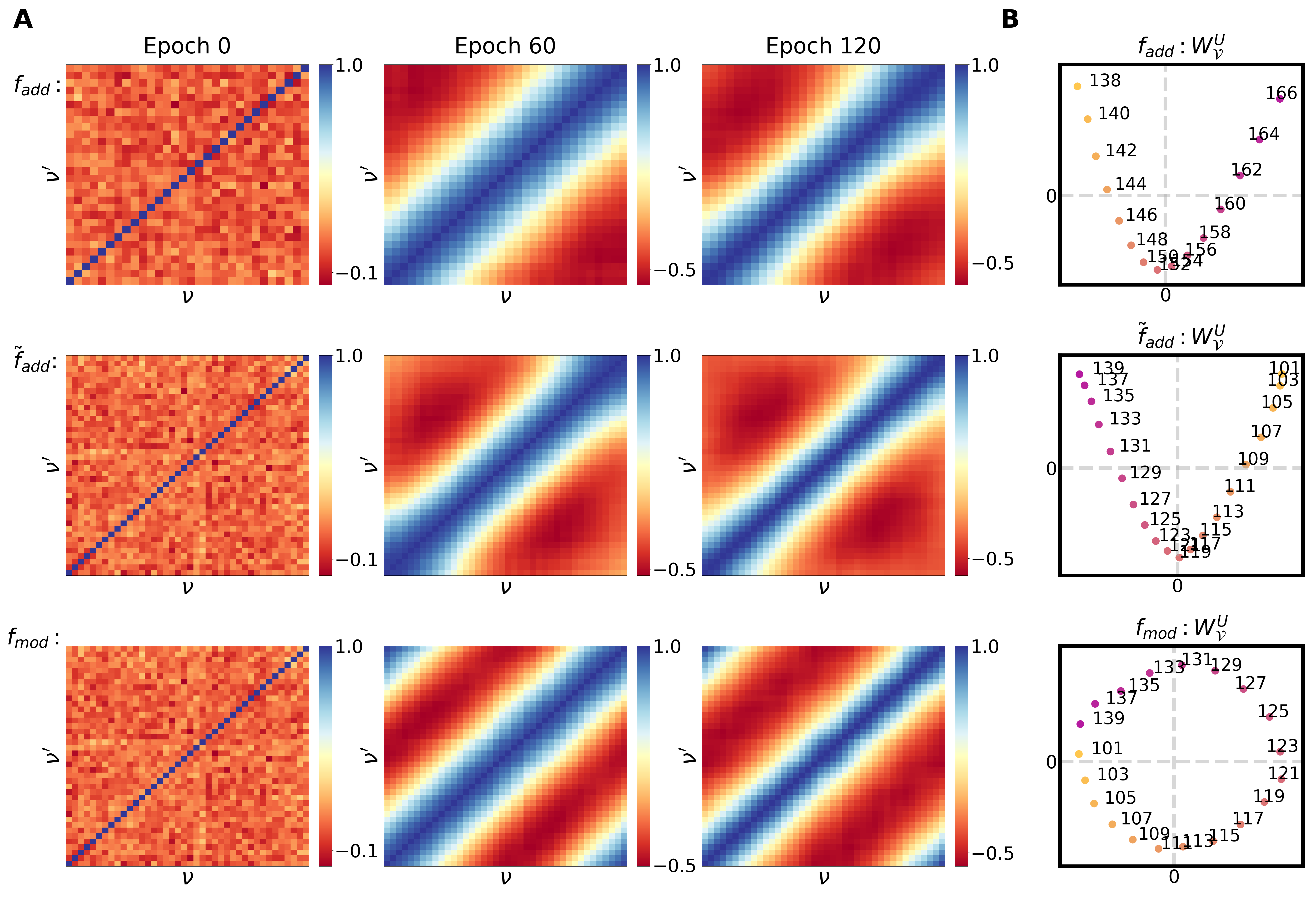}
    \caption{A: The heatmap of the $\cos\left(\vW^U_{\fV}\right)$ with label index in $F_{\rm ffn}$ during the training process. B: PCA projection of  $\vW^U_{\fV}$ in $F_{\rm ffn}$ (epoch 120).}
    \label{fig:ffn_unembedding}
\end{figure}
\newpage
\section{Language Models}\label{app:qwen}

\subsection{Technical Details}\label{app:lan_tech}

\paragraph{Remark about Figure~\ref{fig:qwen_train} C} In each subset $D_i,i=1,2,\cdots M$, we define the set $\fS_i=\left\{s^i_j\right\}_{j=1}^{C_i}$ as the set of the $C_i$ tokens which appear most frequently in $D_i$. Based on the dataset $D_i$, and denote $\vW^{E_i}$ as the embedding matrix of the model corresponding to dataset $D_i$, we compute that
\begin{equation*}
    \cos_{D_i}\left(\vW^{E}_{s^i_j},\vW^{E}\right) = \left[\cos\left(\vW^{E_i}_{s^i_j},\vW^{E_i}_{s^{\prime}}\right)\right]_{s^\prime\in\fS_i}\in\sR^{C_i},
\end{equation*}
and 
\begin{equation*}
    \cos_{D_i}\left(\vphi^{\rm next}_{s^i_j},\vphi^{\rm next}\right) = \left[\cos\left(\vphi^{\rm next}_{s^i_j},\vphi^{\rm next}_{s^\prime}\right)\right]_{s^\prime\in\fS_i}\in\sR^{C_i}.
\end{equation*}
for any token $s^i_j\in\fS_i$. Then we define the correlation coefficient
\begin{align*}
    R_{D_i}\left(s^i_j\right) = {\rm Corr}\left(\cos_{D_i}\left(\vW^E_{s^i_j},\vW^E\right),\cos_{D_i}\left(\vphi^{\rm next}_{s^i_j},\vphi^{\rm next}\right)\right)
\end{align*}
and the average embedding similarity as
\begin{align*}
    {\rm Mean}_{\vW^E,D_i}\left(s^i_j\right)=\frac{1}{C_i}\cos_{D_i}\left(\vW^E_{s^i_j},\vW^E\right)\cdot \mathbf{1}.
\end{align*}

 Then we concatenate the metrics with all token $s_j^i\in\fS_i,j=1,2,\cdots,C_i$ and all datasets $\fS_i,i=1,2,\cdots,M$, i.e.
\begin{align*}
    &{\rm Corr}\left(\cos\left(\vW^E_s,\vW^E\right),\cos\left(\vphi^{\rm next}_s,\vphi^{\rm next}\right)\right)= \left[R_{D_i}\left(s^i_j\right)\right]_{j=1,2,\cdots,C_i}^{i=1,2,\cdots, M}\in\sR^{\sum_{i=1}^M C_i},\\
    &{\rm Mean}\left(\cos\left(\vW^E_s,\vW^E\right)\right) = \left[{\rm Mean}_{\vW^E,D_i}\left(s^i_j\right)\right]_{j=1,2,\cdots,C_i}^{i=1,2,\cdots, M}\in\sR^{\sum_{i=1}^M C_i}.
\end{align*}
Figure~\ref{fig:qwen_train} displays the relation between ${\rm Corr}\left(\cos\left(\vW^E_s,\vW^E\right),\cos\left(\vphi^{\rm next}_s,\vphi^{\rm next}\right)\right)$ and ${\rm Mean}\left(\cos\left(\vW^E_s,\vW^E\right)\right)$, revealing a positive correlation. In our work, $M=5$, and we set up $C_i=10000$ for each dataset.

\paragraph{Remark about Figure~\ref{fig:qwen_train} D \& E}
In each subset $D_i,i=1,2,\cdots M$, we define the set $\fS_i=\left\{s^i_j\right\}_{j=1}^{C_i}$ as the set of the $C_i$ tokens which appear most frequently in $D_i$. We compute that 
\begin{align*}
    \cos_{D_i}\left(\vW^E\right) = \left[\cos\left(\vW^{E_i}_s,\vW^{E_i}_{s^\prime}\right)\right]_{s,s^\prime\in\fS_i}\in\sR^{C_i\times C_i}
\end{align*}
and 
\begin{align*}
    \cos_{D_i}\left(\vphi^{\rm next}\right) = \left[\cos\left(\vphi^{\rm next}_s,\vphi^{\rm next}_{s^\prime}\right)\right]_{s,s^\prime\in\fS_i}\in\sR^{C_i\times C_i}.
\end{align*}
Then translate the similarity matrix into a percentile formulation, i.e.
\begin{align*}
    p_{\cos_{D_i}\left(\vW^E\right)} = {\rm Percentile}\left(\cos_{D_i}\left(\vW^E\right)\right),\quad
    p_{\cos_{D_i}\left(\vphi^{\rm next}\right)} = {\rm Percentile}\left(\cos_{D_i}\left(\vphi^{\rm next}\right)\right)
\end{align*}
and $p_{\cos\left(\vW^E\right)}=\left[p_{\cos_{D_i}\left(\vW^E\right)}\right]_{i=1,2,\cdots,M},\quad p_{\cos\left(\vphi^{\rm next}\right)}=\left[p_{\cos_{D_i}\left(\vphi^{\rm next}\right)}\right]_{i=1,2,\cdots,M}$. Figure~\ref{fig:qwen_train} D and E reveal the distribution and average value of $p_{\cos\left(\vphi^{\rm next}\right)}$, where $k\times 10 \%\leq p_{\cos\left(\vW^E\right)}<(k+1)\times 10 \%, k=0,1,2,\cdots,9$.
\paragraph{Tied Embedding} In the Qwen2.5-3B-base model, the embedding matrix and unembedding matrix are the same one, which aims for computational source saving. Under this condition, we have that
\begin{align*}
    \frac{d\vW^{E}_{s}}{dt} &=r^{\rm in}_s\vW^{{U},T}\vphi_s^{\rm next} +  r_{s}^{\rm out}\vW^{E}\vvarphi_s^{\rm pre} +\veta\\
    &=\vW^E\left(r^{\rm in}_s\vphi_s^{\rm next}+r_{s}^{\rm out}\vvarphi_s^{\rm pre}\right)+\veta.
\end{align*}
Since the next-token-prediction, each token will be an input and an output, except the last token in a sequence, resulting in $r^{\rm in}_s \approx r^{\rm out}_s$. Denote $r_s=r^{\rm in}_{s}$ and $\tilde{\vphi}_s=\vphi_s^{\rm next}+\vvarphi_s^{\rm pre}$, then we have
\begin{align*}
    \frac{d\vW^{E}_{s}}{dt} = r_s\vW^E\tilde{\vphi}_s+\veta.
\end{align*}

\newpage
\subsection{Complete results}
Figure~\ref{fig:qwen_appendix} represents the cosine similarity distribution of $\vW^E,\vphi^{\rm next},\vW^U$ and $\vvarphi^{\rm pre}$ in the other 4 subsets of Pile we selected, exhibiting an analogous phenomenon with the observation in Figure~\ref{fig:qwen_train}. The distribution representations $\vphi^{\rm next}$ and $\vvarphi^{\rm pre}$ could effectively capture the high similarity among embedding vectors and unembedding vectors. Figure~\ref{fig:qwen_appendix_distribution} displays the completed result of Figure~\ref{fig:qwen_train} D.
\begin{figure}[htbp]
    \centering
    \vspace{-10pt}\includegraphics[width=0.7\linewidth]{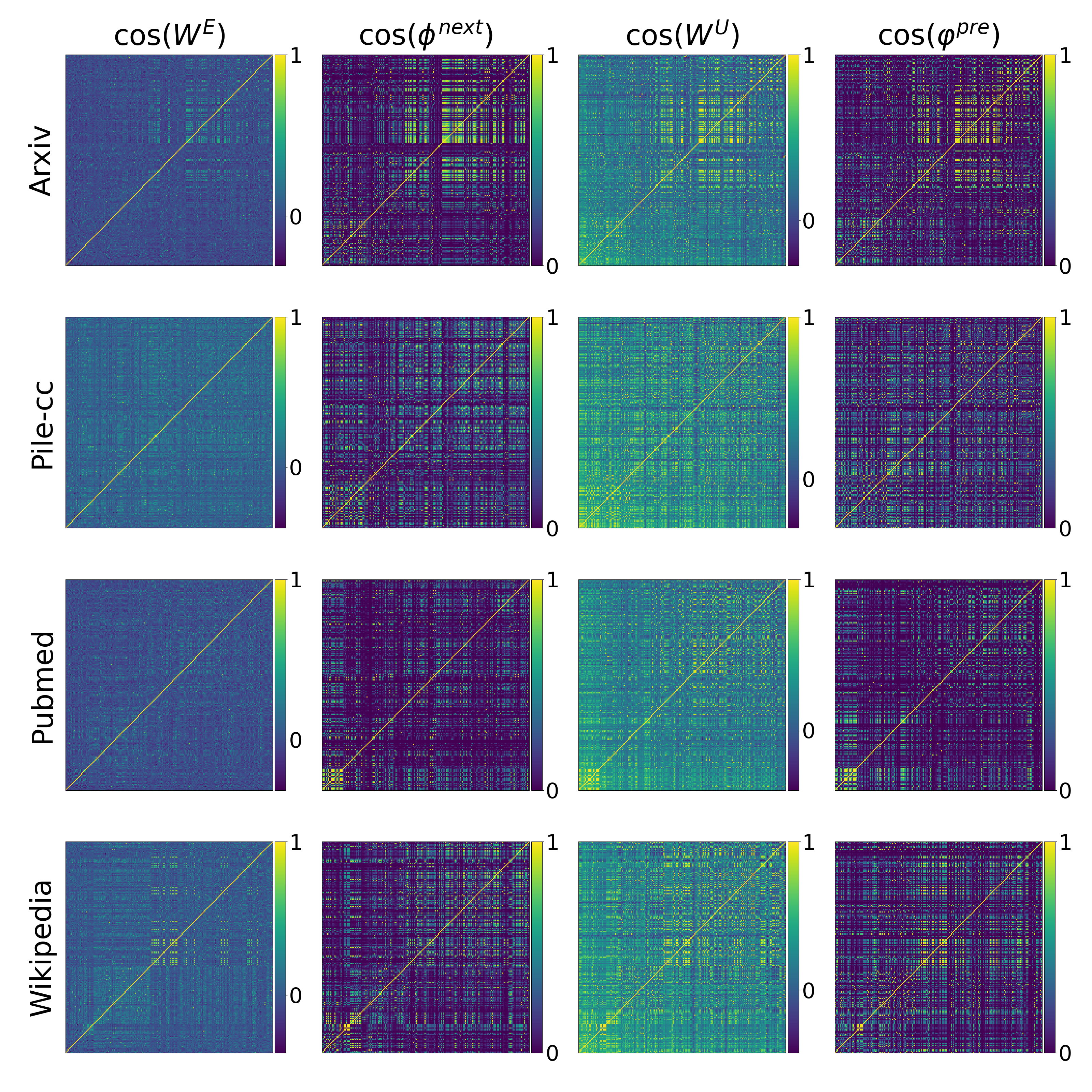}
    \vspace{-10pt}
    \caption{Cosine similarity distribution of $\vW^E,\vphi^{\rm next},\vW^U,\vvarphi^{\rm pre}$ in each experiment with distinct dataset. The tokens displayed are those with the most appearances in the dataset.}
    \vspace{-10pt}
    \label{fig:qwen_appendix}
\end{figure}

\begin{figure}[htbp]
    \centering
    \vspace{-10pt}
    \includegraphics[width=0.7\linewidth]{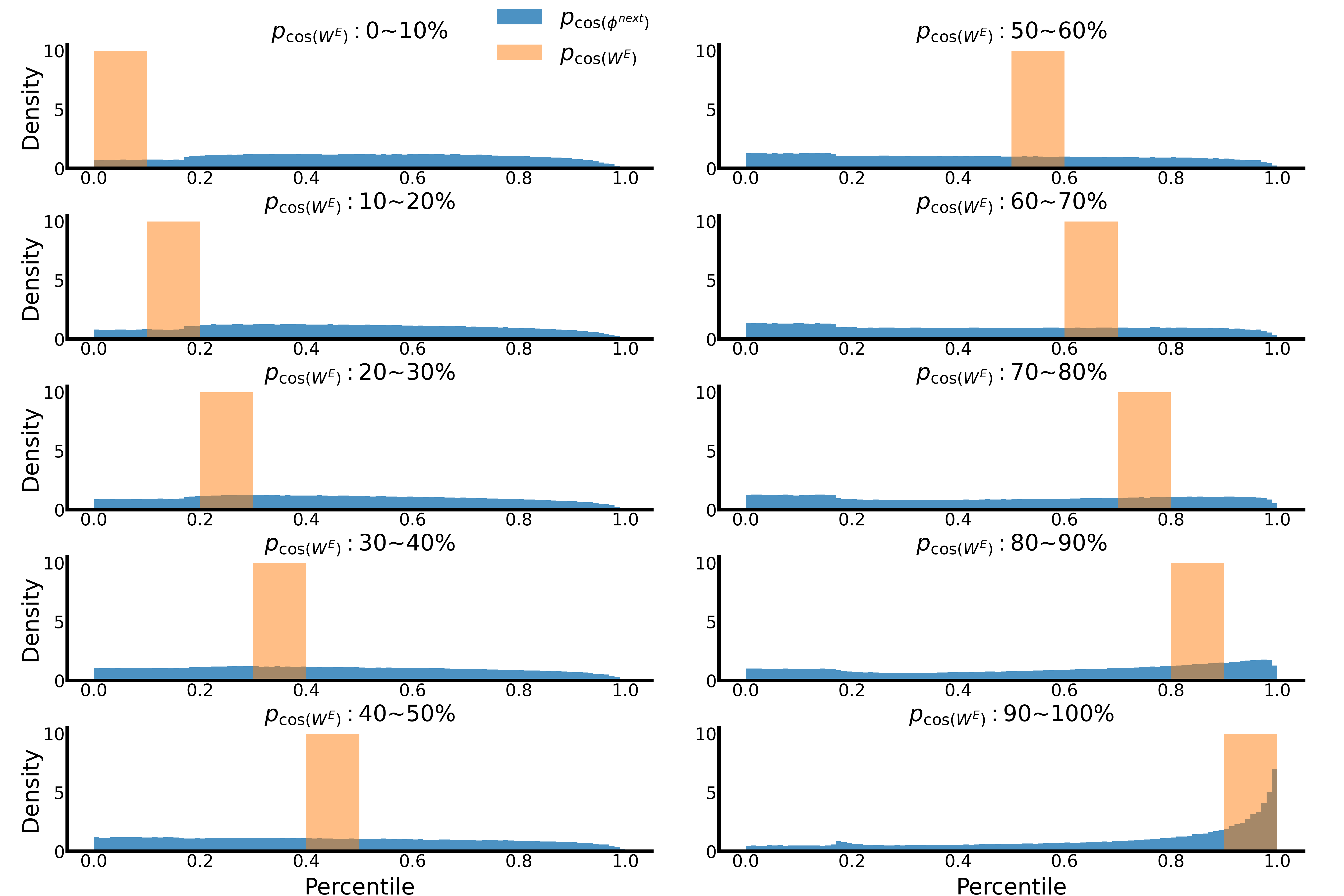}
    \vspace{-10pt}
    \caption{Distribution of $p_{\cos\left(\vphi^{\rm next}\right)}$, conditioned on intervals $0\sim10 \%, 10\sim20 \%,\cdots,  90\sim100 \%$ of the $p_{\cos\left(\vW^E\right)}$.}
    \vspace{-10pt}
    \label{fig:qwen_appendix_distribution}
\end{figure}

\subsection{Case Analysis}
We provide a detailed case to explain the group of tokens exhibiting high embedding similarities. In experiments on the Pile-dm-mathematics dataset, tokens such as ``$/a$'', ``$/b$'', ``$/c$'', and ``$/d$'' often serve as denominators in mathematical expressions. Figure~\ref{fig:qwen_case} shows the cosine similarities of both their embedding vectors and distribution representations, which are notably high for all tokens except ``$/e$'', which does not appear in the dataset. These tokens share highly similar semantics and also exhibit very similar next-token distributions, most frequently followed by ``*'' or ``)''. This similarity in next-token distribution leads to strong similarities in their embedding vectors. This example vividly illustrates how data distribution shapes semantic structure within the embedding space, particularly in the case of tokens with high semantic affinity.
\begin{figure}[htbp]
    \centering
    \includegraphics[width=1\linewidth]{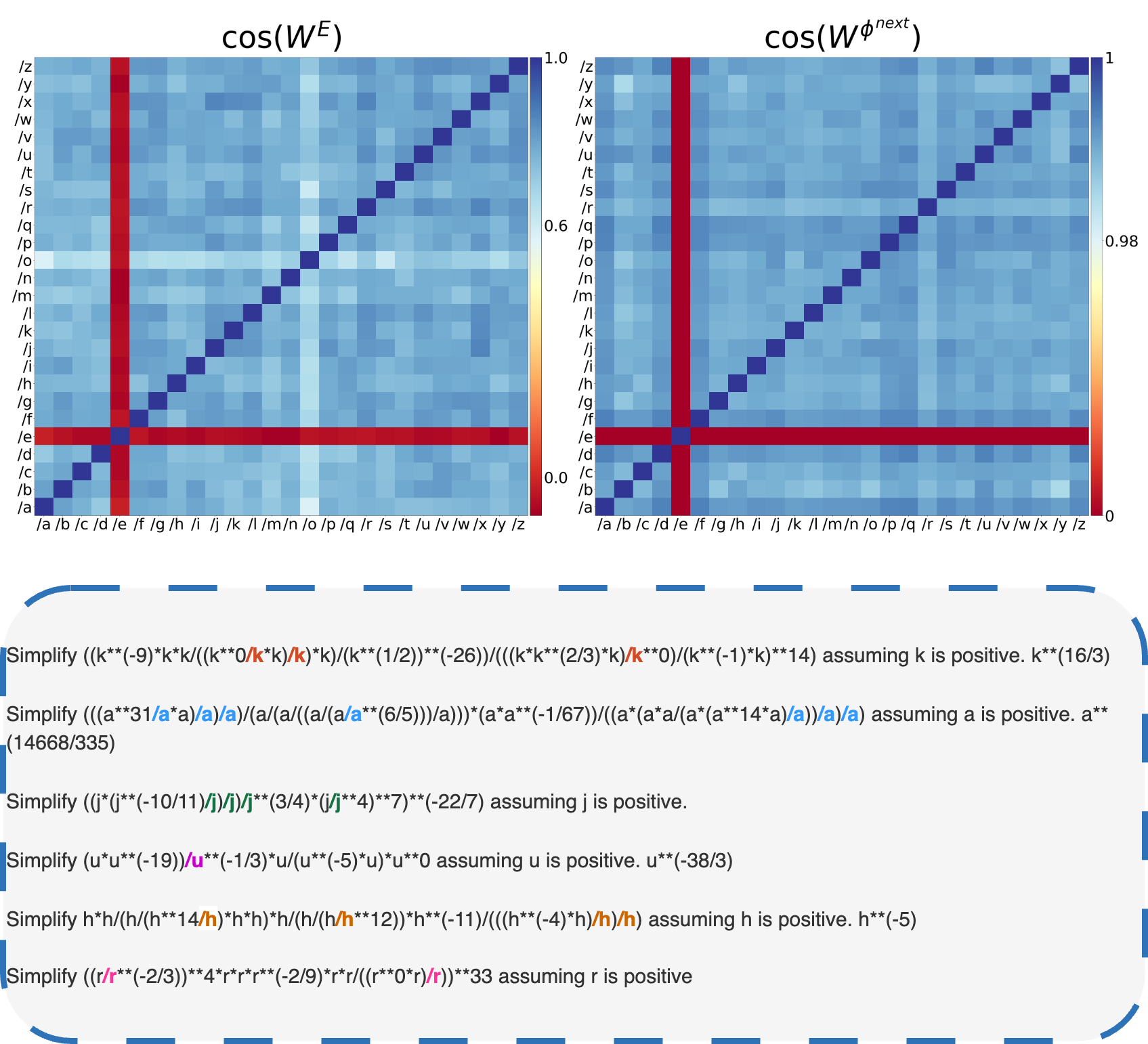}
    \caption{A case analysis of the token group ``$/a$'', ``$/b$'', ``$/c$'', etc. The first row depicts the cosine similarity of their embeddings (left) and distribution representations (right). The second row exhibits the contexts containing these tokens, which are highlighted by different colors.}
    \label{fig:qwen_case}
\end{figure}

\newpage

\subsection{Results of LLama 2}\label{app:llama}
To assess the generalizability of our analysis in Section~\ref{sec:language_model} across different model architectures and tokenizers, we replicate the experiment using the Llama 2 architecture. We employ the same dataset from Pile, and the training configurations are the same as the experiments of Qwen2.5. As shown in Figure~\ref{fig:llama_train}, the probability signatures effectively capture structural relationships in the embedding space, especially in regions exhibiting high embedding similarity. These results align closely with those in Figure~\ref{fig:qwen_train}, indicating that our analytical approach is robust to variations in model architecture.
\begin{figure}[htbp]
    \centering
    \includegraphics[width=1\linewidth]{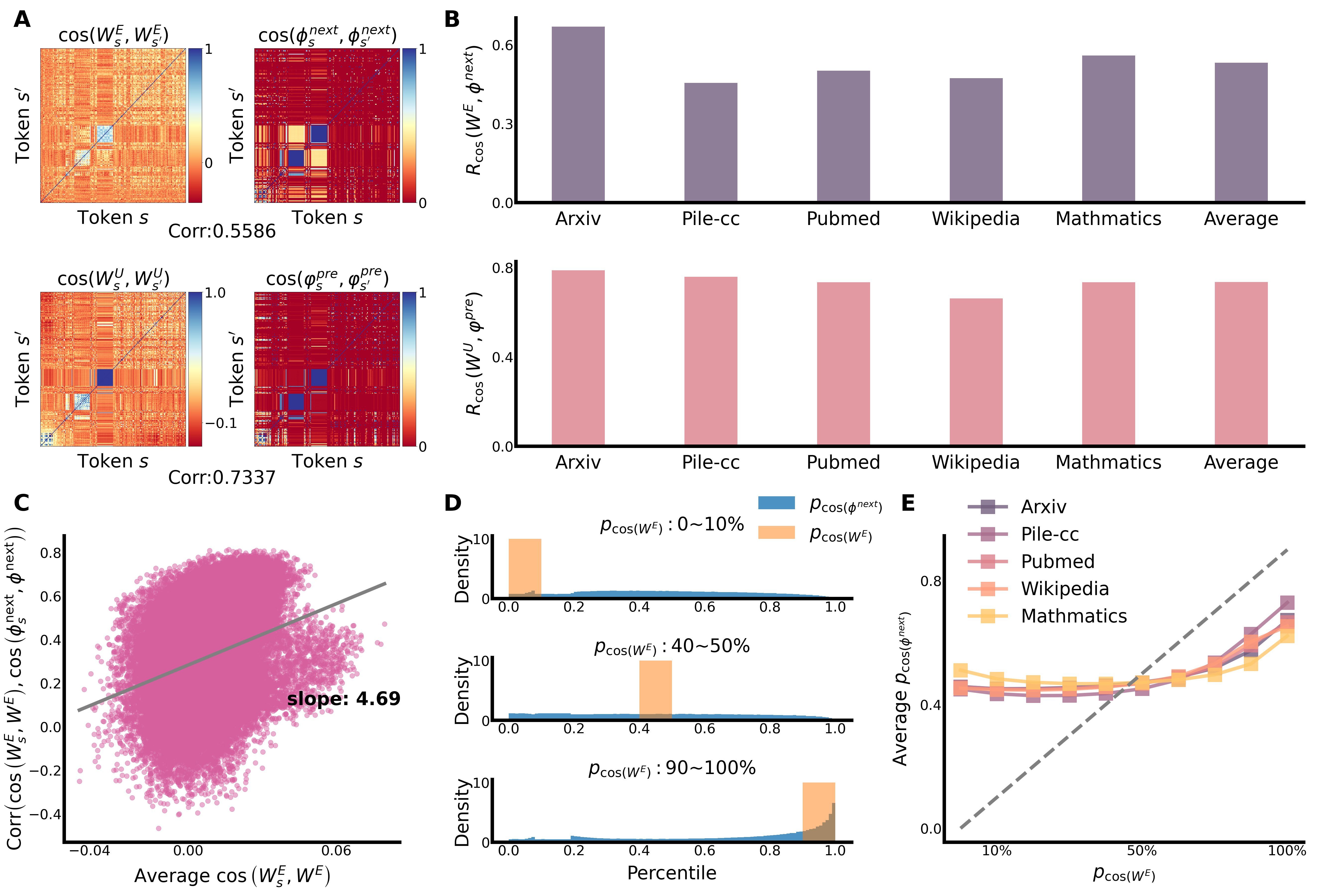}
    \caption{Results with Llama-2 architecture. A: Heatmap of the cosine similarity of $\vW^E,\vW^U,\vphi^{\rm next}$ and $\vvarphi^{\rm pre}$. B: $R_{\cos}\left(\vW^E,\vphi^{\rm next}\right)$ (top) and $R_{\rm cos}\left(\vW^U,\vvarphi^{\rm pre}\right)$ (bottom) with different datasets. C: Relation between ${\rm Corr}\left(\cos\left(\vW^E_s,\vW^E\right),\cos\left(\vphi_s^{\rm next},\vphi^{\rm next}\right)\right)$ and the average value of $\cos\left(\vW^E_s,\vW^E\right)$. Each point denotes a token $s$. D: Distribution of $p_{\cos\left(\vphi^{\rm next}\right)}$, conditioned on intervals $0\sim10 \%, 40\sim50 \%$ and $90\sim100 \%$ of the $p_{\cos\left(\vW^E\right)}$. E: Average value of $p_{\cos\left(\vphi^{\rm next}\right)}$ within each interval of $p_{\cos\left(\vW^E\right)}$.}
    \label{fig:llama_train}
\end{figure}

\newpage

\section{Theoretical Details}
\subsection{Proof of Proposition~\ref{prop:dynamics_emb}}

\begin{lemma}\label{lem:softmax}
    Given a model $F$ and data pair $\left(\vX,y\right)\in\sN^{+,L}\times \sN^{+}$, $\ell=-\log\text{Softmax}\left(F\left(\vX\right)\right)_{y}$, we have that
    \begin{equation}
        \frac{\partial \ell}{\partial F\left(\vX\right)}=\vp-\ve_{y},
    \end{equation}
    where $\vp=\text{softmax}\left(\vX\right).$
\end{lemma}
\begin{proof}
    It's noted $\ell = -F\left(\vX\right)_y+\log \sum_{j=1}^{d_{\rm vob}}\exp F\left(\vX\right)_j$, then we have 
    \begin{align*}
        \frac{\partial \ell}{\partial F\left(X\right)_i}= -\delta_{i=y}+\frac{\exp F\left(\vX\right)_i}{\sum_{j=1}^{d_{\rm vob}}\exp F\left(\vX\right)_j}=\vp_i-\delta_{i=y},
    \end{align*}
    where $\delta_{i=y}=1$ if $i=y$ else $0$. This indicates that $\frac{\partial \ell}{\partial F\left(\vX\right)}=\vp-\ve_{y}$.
\end{proof}

With Lemma~\ref{lem:softmax}, we could obtain the derivative of $\ell$ with respect to $\vW^E_x$ for any $x\in\fV$ as follows:
\begin{align*}
    \frac{\partial \ell^i}{\partial \vW^E_x}&=\frac{\partial F\left(\vX^i\right)}{\partial\vW^E_{x}} \frac{\partial \ell^i}{\partial F\left(\vX^i\right)}\\
    &= \left(\vW^{U,T}\left(\vp^i-\ve_{y^i}\right)\right)\odot G^{\left(1\right)}\left(\vW^E_{\vX^i}\right).
\end{align*}
Then the gradient flow of $vW^E_x$ could be obtained by
\begin{align*}
    \frac{d\vW^E_x}{dt} =& -\frac{1}{N}\sum_{i=1}^N\frac{\partial \ell^i}{\partial \vW^E_x}=\frac{1}{N}\sum_{i=1}^N\left(\vW^{U,T}\left(\vp^i-\ve_{y^i}\right)\right)\odot G^{\left(1\right)}\left(\vW^E_{\vX^i}\right),
\end{align*}
Since ${\rm diag}\left(G^{\left(1\right)}\left(\vW^E_{\vX^i}\right)\right)=0$ if $x\notin \vX^i$, we have that
\begin{align*}
    \frac{d\vW^E_x}{dt}&=\frac{1}{N}\sum_{i=1}^{N_x^{\rm in}}\left(\vW^{U,T}\left(\ve_{y^i_x}-\vp^i_x\right)\right)\odot G^{\left(1\right)}\left(\vW^E_{\vX^i_x}\right)\\
    &=\frac{r_x^{\rm in}}{N_x^{\rm in}}\sum_{i=1}^{N_x^{\rm in}}\left(\vW^{U,T}\left(\ve_{y^i_x}-\vp^i_x\right)\right)\odot G^{\left(1\right)}\left(\vW^E_{\vX^i_x}\right).
\end{align*}
Since that $y^i_x$ takes value $\nu\in\fV$, we can rewrite this formation as
\begin{align*}
     \frac{d \vW^{E}_x}{dt} =\sum_{\nu\in\fV}\frac{r_{x,\nu}}{N_{x,\nu}}\left(\vW^{{U}, T}\ve_{\nu}\right)\odot\sum_{i=1}^{N_{x,\nu}}G^{(1)}\left(\vW^E_{\vX_{\left(x,\nu\right)}^i}\right) -\frac{r^{\rm in}_x}{N^{\rm in}_x}\sum_{i=1}^{N^{\rm in}_x}G^{(1)}\left(\vW^E_{\vX^i_x}\right)\odot\left(\vW^{{U}, T}\vp_{x}^i\right).
\end{align*}

\subsection{Proof of Proposition~\ref{prop:dynamics_unemb}}
Similar with the analysis of $\vW^E_x$, we derive the gradient flow of $\vW^U_\nu$ as follows:

\begin{align*}
    \frac{d\vW^U_\nu}{dt} =& -\frac{1}{N}\sum_{i=1}^{N}\frac{\partial \ell^i}{\partial \vW^U_{\nu}}\\
    =& \frac{1}{N}\sum_{i=1}^N \left(\ve_{y^{i,\nu}}-\vp^{i,\nu}\right)\left[G\left(\vW^E_{\vX^i}\right)\right]^T.
\end{align*}
Since $\ve_{y^{i,\nu}}=1$ if $y^i=\nu$ else 0, we have that
\begin{align*}
    \frac{d\vW^U_\nu}{dt} =\frac{r^{\rm out}_{\nu}}{N^{\rm out}_{\nu}}\sum_{i=1}^{N_{\nu}^{\rm out}} \left[G\left(\vW^E_{\vX^i_{\left(\cdot,\nu\right)}}\right)\right]^T - \frac{1}{N}\sum_{i=1}^N\vp^{i,\nu}\left[G\left(\vW^E_{\vX^i}\right)\right]^T.
\end{align*}

\subsection{Proof of Corollary~\ref{cor:emb_linear}}
With proposition~\ref{prop:dynamics_emb}, we have that 
\begin{align*}
    \frac{d \vW^{E}_{\alpha}}{dt} &=\vW^{{U}, T}\left(\sum_{\nu\in\fV}r_{\alpha,\nu}\ve_{\nu} -\frac{r^{\rm in}_\alpha}{N^{\rm in}_\alpha}\sum_{i=1}^{N^{\rm in}_x}\vp_{\alpha}^i\right)\\
    &= \vW^{U,T}r^{\rm in}_\alpha\left(\sum_{\nu\in\fV}\frac{r_{\alpha,\nu}}{r_{\alpha}^{\rm in}}\ve_{\nu} -\frac{1}{N^{\rm in}_\alpha}\sum_{i=1}^{N^{\rm in}_x}\vp_{\alpha}^i\right).
\end{align*}
Utilizing that ${\rm softmax}\left(\vf\right)= \frac{1}{d_{\rm vob}}\mathbf{1}+\frac{1}{d_{\rm vob}}\vf+\mathcal{O}\left(d_{\rm vob}^{-2}\vf\right)$, we obtain that
\begin{align*}
    \frac{d \vW^{E}_{\alpha}}{dt} &= \vW^{U,T}r^{\rm in}_\alpha\left(\sum_{\nu\in\fV}\frac{r_{\alpha,\nu}}{r_{\alpha}^{\rm in}}\ve_{\nu} -\frac{1}{N^{\rm in}_\alpha}\sum_{i=1}^{N^{\rm in}_\alpha}\left(\frac{1}{d_{\rm vob}}\mathbf{1}-\frac{1}{d_{\rm vob}}\vW^U\left(\vW^E_{z_i}+\vW^E_{\alpha_i}+\vW^E_{\alpha}\right)+\mathcal{O}\left(d_{\rm vob}^{-2}\vW^U\vW^E_\alpha\right)\right)\right)\\
    &= \vW^{U,T}r^{\rm in}_\alpha\left(\sum_{\nu\in\fV}\frac{r_{\alpha,\nu}}{r_{\alpha}^{\rm in}}\ve_{\nu} -\frac{1}{d_{\rm vob}}\mathbf{1} +\frac{1}{d_{\rm vob}}\vW^U\left(\frac{1}{N^{\rm in}_\alpha}\sum_{i=1}^{N^{\rm in}_\alpha}\left(\vW^E_{z_i}+\vW^E_{\alpha_i}\right)+\vW^E_{\alpha}\right)+\mathcal{O}\left(d_{\rm vob}^{-2}\vW^U\vW^E_\alpha\right)\right)\\
    &= \vW^{U,T}r^{\rm in}_\alpha\left(\sum_{\nu\in\fV}\frac{r_{\alpha,\nu}}{r_{\alpha}^{\rm in}}\ve_{\nu} -\frac{1}{d_{\rm vob}}\mathbf{1} +\frac{1}{d_{\rm vob}}\vW^U\left(\sum_{x\in\left(\fZ\cap\fA\right)}\frac{N_{\alpha,x}^{\rm in}}{N_{\alpha}^{\rm in}}\vW^E_{x}+\vW^E_{\alpha}\right)+\mathcal{O}\left(d_{\rm vob}^{-2}\vW^U\vW^E_\alpha\right)\right).
\end{align*}

Let $N\rightarrow\infty$, we have that
\begin{align*}
     \frac{d \vW^{E}_{\alpha}}{dt} =&\vW^{U,T}r^{\rm in}_\alpha\left(\sum_{\nu\in\fV}\mathbb{P}_{\pi}\left(y=\nu\mid \alpha\in\vX\right)\ve_{\nu}-\frac{1}{d_{\rm vob}}\mathbf{1}\right.\\ &\left.+\frac{1}{d_{\rm vob}}\vW^U\left(\sum_{x\in\left(\fZ\cap\fA\right)}\mathbb{P}_\pi\left(x\in\vX\mid \alpha\in\vX\right)\vW^E_{x}+\vW^E_{\alpha}\right)+\mathcal{O}\left(d_{\rm vob}^{-2}\vW^U\vW^E_\alpha\right)\right)\\
     =& \vW^{U,T}r_\alpha^{\rm in}\left(\vphi_\alpha^{y}+\frac{1}{d_{\rm vob}}\vW^U\vW^{E}\vphi_{\alpha}^{\vX}\right) + \veta,
\end{align*}
where $\veta=\vW^{U,T}r_{\alpha}^{\rm in}\left(\frac{1}{d_{\rm vob}}\left(\vW^U\vW^E_\alpha-\mathbf{1}\right)+\mathcal{O}\left(d_{\rm vob}^{-2}\vW^U\vW^E_\alpha\right)\right)$ contains the higher-order term and the data independent term.

\subsection{Proof of Corollary~\ref{cor:emb_nonlinear}}\label{proof:emb_nonlinear}
\begin{proof}
    Since the small initialization, we assume that the activation function can be approximated by the following form with the Weierstrass approximation theorem.
    \begin{align*}
        \sigma\left(\sum_{x\in\vX}\vW^E_{x}\right)=C_0+ C_1\left(\sum_{x\in\vX}\vW_x^{E}\right)+C_2\left(\sum_{x\in\vX}\vW_x^{E}\right)^{\odot 2}+\vepsilon.
    \end{align*}
    With the loss of the generalization, we assume that $C_0=0,C_1=1,C_2=\frac{1}{2}$. Then we have
    \begin{align*}
        \frac{d \vW^{E}_{\alpha}}{dt} =& \underbrace{\sum_{\nu\in\fV} \frac{r_{\alpha,\nu}}{N_{\alpha,\nu}}\left(\vW^{{U}, T}\ve_{\nu}\right)\odot\sum_{i=1}^{N_{\alpha,\nu}}\left(\mathbf{1}+\sum_{x\in\vX_{(\alpha,\nu)}^i}\vW^E_{\vX_{(\alpha,\nu)}^i}\right)}_{\vJ^{y}} \\
        &-\underbrace{\frac{r^{\rm in}_\alpha}{N^{\rm in}_\alpha}\sum_{i=1}^{N^{\rm in}_\alpha}\left(\mathbf{1}+\sum_{x\in\vX_{\alpha}^i}\vW^E_{\vX_{\alpha}^i}\right)\odot\left(\vW^{{U}, T}\vp_{\alpha}^i\right)}_{\vJ^p}.\\
    \end{align*}
    For the term $\vJ^y$ we have 
    \begin{align*}
        \vJ^y&=\vW^{{U}, T}\sum_{\nu\in\fV} r_{\alpha,\nu}\ve_{\nu} + \sum_{\nu\in\fV} r_{\alpha,\nu}\left(\vW^{{U}, T}\ve_{\nu}\right)\odot \left(\vW^{E}_{\alpha}+\frac{1}{2N_{\alpha,\nu}}\sum_{i=1}^{2N_{\alpha,\nu}}\vW_{x^i_{(\alpha,\nu)}}^E\right).
    \end{align*}
    Let $N\rightarrow\infty$, we have that
    \begin{align*}
        \vJ^y &= \vW^{U,T}r^{\rm in}_{\alpha}\vphi_{\alpha}^y \odot\left(\mathbf{1}+\vW^E_{\alpha}\right) + \sum_{\nu\in\fV}{\rm diag}\left(\vW^{U}_{\nu}\right)r_{\alpha,\nu}\sum_{x^\prime\in\fV}\mathbb{P}\left(x^{\prime}\in\vX\mid \alpha \in\vX,y=\nu\right)\vW^E_{x^\prime}\\
        &= \vW^{U,T}r^{\rm in}_{\alpha}\vphi_{\alpha}^y \odot\left(\mathbf{1}+\vW^E_{\alpha}\right) + \sum_{\nu\in\fV}{\rm diag}\left(r_{\alpha,\nu}\vW^{U}_{\nu}\right)\vW^E\left(\vphi_{\alpha}^{\vX\mid y,T}\right)_{\nu}^T\\
        &= \vW^{U,T}r^{\rm in}_{\alpha}\vphi_{\alpha}^y \odot\left(\mathbf{1}+\vW^E_{\alpha}\right) + \mathbb{T}\cdot \left(\vphi_{\alpha}^{\vX\mid y}\right)^T,
    \end{align*}
    where $\mathbb{T}\in\sR^{d\times d_{\rm vob}\times d_{\rm vob}}$, $\mathbb{T}_{:,:,\nu}=r_{\alpha,\nu}{\rm diag}\left(\vW^{U}_\nu\right)\vW^{E}$ for $\nu\in\fV$ and 0 otherwise.

    Similarly, for the term $\vJ^p$, we have that
    \begin{align*}
        \vJ^p =& \vW^{U,T}r^{\rm in}_\alpha\left(\frac{1}{d_{\rm vob}}\mathbf{1} -\frac{1}{d_{\rm vob}}\vW^U\left(\left(\vW^{E}-{\rm diag}\left(\vW^{U,T}\mathbf{1}\right)\right)\vphi_{\alpha}^{\vX}+\vW^E_{\alpha}\right)+\vepsilon\right),
    \end{align*}
    where $\vepsilon=\mathcal{O}\left(\frac{1}{d^2_{\rm vob}}\vW^U\vW^E_{\alpha}\right)$. Then we have that
    \begin{align*}
        \frac{d \vW^{E}_{\alpha}}{dt}= \mathbb{T}\cdot \left(\vphi_{\alpha}^{\vX\mid y}\right)^T + \veta_{\vphi_{\alpha}^{y}}+ \frac{1}{d_{\rm vob}}\veta_{\vphi_{\alpha}^{\vX}},
    \end{align*}
    where $\veta_{\vphi_{\alpha}^{y}}= \vW^{U,T}r^{\rm in}_{\alpha}\vphi_{\alpha}^y \odot\left(\mathbf{1}+\vW^E_{\alpha}\right)$, $\veta_{\vphi_{\alpha}^{\vX}}=d_{\rm vob}\vJ^p$.
\end{proof}

\subsection{Proof of Corollary~\ref{cor:unemb_linear}}
\begin{proof}
    With Proposition~\ref{prop:dynamics_unemb}, we have that
    \begin{align*}
        \frac{d \vW^{U}_{\nu}}{dt} =& \frac{r^{\rm out}_{\nu}}{N^{\rm out}_{\nu}}\sum_{i=1}^{N^{\rm out}_{\nu}}\left (\sum_{x\in \vX_{\left(\cdot,\nu\right)}^i}\vW^{E}_{x}\right)^T-\frac{1}{N}\sum_{i=1}^N\vp^{i,\nu} \left (\sum_{x\in \vX^i}\vW^{E}_{x}\right)^T\\
        =& Lr^{\rm out}_{\nu}\sum_{x\in\fV}\mathbb{P}_\pi\left(x\in\vX\mid y=\nu\right)\vW^{E,T}_x - L\sum_{x\in\fV}\mathbb{E}_{\pi}\left[\vp^{\nu}\mid x\in\vX\right]\vW^{E,T}_{x}\\
        =& Lr^{\rm out}_{\nu}\left(\vW^{E}\vvarphi_{\nu}^{\vX}\right)^T - \veta,
    \end{align*}
    where $\veta=L\left(\vW^{E}\mathbb{E}_{\pi}\left[\vp\mid x\in\vX\right]\right)^T$.
\end{proof}

\subsection{Proof of Corollary~\ref{cor:lan}}
\begin{proof}
    The next-token-prediction training loss could be formulated as
   \begin{equation*}
    \ell^i = \frac{1}{L}\sum_{t=1}^{L-1}{\rm CrossEntropy}\left(F_{\rm lan}\left(\vX_{:t}\right);\ve_{\vX_{t+1}}\right). 
    \end{equation*}
    So we have that 
    \begin{align*}
        \frac{\partial \ell^i}{\partial \vW^E_s} = \frac{1}{L}\sum_{t=1}^{L-1}\vW^{U,T}\left(\vp^i_t-\ve_{\vX^i_{t+1}}\right)\odot\left(\delta_{\vX^i_{t}=s}\mathbf{1}+\tilde{F}^{\left(1\right)}\left(\vX^i_{:t}\right)\right).
    \end{align*}
    Furthermore, we have that
    \begin{align*}
        \frac{d\vW^E_s}{dt} =& \frac{1}{NL}\sum_{i=1}^N\sum_{t=1}^{L-1}\vW^{U,T}\left(\ve_{\vX^i_{t+1}}-\vp_t^i\right)\odot\left(\delta_{\vX^i_{t}=s}\mathbf{1}+\tilde{F}^{\left(1\right)}\left(\vX^i_{:t}\right)\right)\\
        =&\frac{1}{NL}\vW^{U,T}\sum_{i=1}^N\sum_{t=1}^{L-1}\delta_{\vX^i_{t}=s}\ve_{\vX^i_{t+1}}+\frac{1}{NL}\vW^{U,T}\sum_{i=1}^N\sum_{t=1}^{L-1}\ve_{\vX^i_{t+1}}\odot \tilde{F}^{\left(1\right)}\left(\vX^i_{:t}\right)\\
        &- \frac{1}{NL}\sum_{i=1}^N\sum_{t=1}^{L-1}\vW^{U,T}\vp_t^i\odot \left(\delta_{\vX^i_{t}=s}\mathbf{1}+\tilde{F}^{\left(1\right)}\left(\vX^i_{:t}\right)\right).
    \end{align*}
    Since the small initialization, assuming that $||\vW||_{\infty}= \mathcal{O}\left(d^{-\gamma}\right)$ for any trainable parameter matrix $\vW$, we have that $||\tilde{F}^{\left(1\right)}\left(\vX^i_{:t}\right)||_{\infty}=\mathcal{O}\left(d^{1-2\gamma}\right)$ in the initial stage. Let $N\rightarrow\infty$, we have that
    \begin{align*}
        \frac{d\vW^E_s}{dt} = r^{\rm in}_{s}\vW^{U,T}\left(\vphi^{\rm next}_s - \veta^E\right),
    \end{align*}
    where $\veta^E = \sum_{t=1}^{L-1}\mathbb{E}_{\pi}\left[\vp\mid \vX_t=s\right] + \mathcal{O}\left(d^{1-2\gamma}\vphi^{\rm next}_s\right)$.
    Similarly, we have that
    \begin{align*}
        \frac{d\vW^U_{s}}{dt} = \frac{1}{NL}\sum_{i=1}^N\sum_{t=1}^{L-1}\left(\delta_{\vX^i_{t+1}=s}-\vp^{i,s}_{\vX^i_{:t}}\right)\left(\vW^{E,T}_{\vX^{i}_{t}}+\tilde{F}\left(\vX^i_{:t}\right)^T\right),
    \end{align*}
    where $\vp^{i,s}_{\vX^i_{:t}}$ means the $s$-th element of the output probability with input sequence $\vX^i_{:t}$. Let $N\rightarrow\infty$, we have 
    \begin{align*}
        \frac{d\vW^U_{s}}{dt} = r^{\rm out}_s\left(\vW^E\vvarphi^{\rm pre}_s\right)^T+\veta^U,
    \end{align*}
    where $\veta^U=\sum_{t=1}^{L-1}\mathbb{E}_{\pi}\left[\vp^s_{\vX_{:t}}\vW^{E,T}_{\vX_{t}}\right]+\mathcal{O}\left(r^{\rm out}_sd^{1-2\gamma}\left(\vW^{E}\vvarphi^{\rm pre}_s\right)^T\right)$.
\end{proof}

\end{document}